\documentclass{article}

\PassOptionsToPackage{numbers,sort}{natbib}
\usepackage[preprint]{neurips_2026}

\usepackage[utf8]{inputenc} %
\usepackage[T1]{fontenc}    %
\usepackage{hyperref}       %
\usepackage{url}            %
\usepackage{booktabs}       %
\usepackage{amsfonts}       %
\usepackage{nicefrac}       %
\usepackage{microtype}      %
\usepackage{xcolor}         %

\usepackage{adjustbox}      %
\usepackage{graphicx,color}
\usepackage{siunitx}        %
\usepackage{wrapfig}        %
\usepackage{xifthen}
\usepackage{algorithm}
\usepackage{listings}
\usepackage{tablefootnote} %

\usepackage{titlesec}
\titlespacing*{\paragraph}{0pt}{0.0ex plus 0.2ex minus 0.1ex}{1ex plus 0.05ex} %

\usepackage[noend]{algpseudocode}
\algrenewcommand\algorithmicindent{1em}
\algrenewcommand{\algorithmiccomment}[1]{%
\bgroup\hskip2em\textcolor{ourdarkgreen}{//~\textsl{#1}}\egroup}

\usepackage[textsize=tiny]{todonotes}

\usepackage{colortbl}
\usepackage{tabulary}
\usepackage{multirow}

\newcommand{\tablestyle}[2]{\setlength{\tabcolsep}{#1}\renewcommand{\arraystretch}{#2}\centering
\small
}
\newcolumntype{x}[1]{>{\centering\arraybackslash}p{#1pt}}
\newcolumntype{y}[1]{>{\raggedright\arraybackslash}p{#1pt}}
\newcolumntype{z}[1]{>{\raggedleft\arraybackslash}p{#1pt}}

\newcommand\tablefontsize{%
  \fontsize{8pt}{9pt}\selectfont
}

\definecolor{codebg}{rgb}{0.97,0.97,0.97}
\definecolor{codeblue}{rgb}{0.25,0.5,0.5}
\definecolor{codekw}{rgb}{0.85, 0.18, 0.50}
\definecolor{codesign}{RGB}{0, 0, 255}
\definecolor{codefunc}{rgb}{0.85, 0.18, 0.50}

\lstdefinelanguage{PythonFuncColor}{
  language=Python,
  keywordstyle=\color{blue}\bfseries,
  commentstyle=\color{codeblue},
  stringstyle=\color{orange},
  showstringspaces=false,
  basicstyle=\ttfamily\small,
  literate=
    {True}{{\color{blue}True}}{1}
    {False}{{\color{blue}False}}{1}
    {inf}{{\color{blue}inf}}{1}
    {None}{{\color{blue}None}}{1}
    {kl_loss}{{\color{codefunc}kl\_loss}}{1}
    {gauss_log_kde}{{\color{codefunc}gauss\_log\_kde}}{1}
    {logsumexp}{{\color{codefunc}logsumexp}}{1}
    {detach}{{\color{codefunc}detach}}{1}
    {fill_diagonal_}{{\color{codefunc}fill\_diagonal\_}}{1}
    {cdist}{{\color{codefunc}cdist}}{1}
    {model}{{\color{codefunc}model}}{1}
    {compute_V}{{\color{codefunc}compute\_V}}{1}
    {randn}{{\color{codefunc}randn}}{1}
    {stopgrad}{{\color{codefunc}stopgrad}}{1}
    {mean}{{\color{codefunc}mean}}{1}
    {cat}{{\color{codefunc}cat}}{1}
    {eye}{{\color{codefunc}eye}}{1}
    {affinity}{{\color{codefunc}affinity}}{1}
    {softmax}{{\color{codefunc}softmax}}{1}
    {mse_loss}{{\color{codefunc}mse\_loss}}{1}
    {sqrt}{{\color{codefunc}sqrt}}{1}
    {split}{{\color{codefunc}split}}{1}
    {sum}{{\color{codefunc}sum}}{1}
    {concat}{{concat} }{1}
}

\usepackage{enumitem}

\usepackage{amsfonts}       %
\usepackage{amsmath}       %
\usepackage{amssymb}
\usepackage{amsthm}
\usepackage{xspace}
\usepackage{thmtools} %

\usepackage{placeins}

\usepackage{xr-hyper}
\usepackage{hyperref}%
\usepackage[capitalize]{cleveref}  %

\makeatletter
\DeclareRobustCommand\onedot{\futurelet\@let@token\@onedot}
\def\@onedot{\ifx\@let@token.\else.\null\fi\xspace}
\makeatother

\newcommand{\eg}{e.g\onedot}

\makeatletter
\newcommand*{\addFileDependency}[1]{%
  \typeout{(#1)}
  \@addtofilelist{#1}
  \IfFileExists{#1}{}{\typeout{No file #1.}}
}
\makeatother

\definecolor{ourblue}{rgb}{0.368,0.507,0.71}    %
\definecolor{ourdarkblue}{HTML}{354d72}
\definecolor{ourlightblue}{HTML}{b4c5dc}

\definecolor{ourorange}{rgb}{0.881,0.611,0.142} %
\definecolor{ourdarkorange}{HTML}{8f6213}
\definecolor{ourlightorange}{HTML}{efc785}

\definecolor{ourgreen}{rgb}{0.56,0.692,0.195}   %
\definecolor{ourdarkgreen}{HTML}{677f24}
\definecolor{ourlightgreen}{HTML}{cee298}

\definecolor{ourred}{rgb}{0.923,0.386,0.209}    %
\definecolor{ourdarkred}{HTML}{94300f}
\definecolor{ourlightred}{HTML}{f7c5b2}

\definecolor{ourviolet}{rgb}{0.528,0.471,0.701} %
\definecolor{ourdarkviolet}{HTML}{463b68}
\definecolor{ourlightviolet}{HTML}{bcb3d4}

\definecolor{ourbrown}{rgb}{0.772,0.432,0.102}  %
\definecolor{ourdarkbrown}{HTML}{905113}
\definecolor{ourlightbrown}{HTML}{f1c093}

\definecolor{ourazure}{rgb}{0.364,0.619,0.782}  %
\definecolor{ourdarkazure}{HTML}{2a5b79}
\definecolor{ourlightazure}{HTML}{a6cae0}

\definecolor{ourolive}{rgb}{0.572,0.586,0.}     %
\definecolor{ourdarkolive}{HTML}{5b5c00}
\definecolor{ourlightolive}{HTML}{e3e395}

\definecolor{ourgray}{RGB}{102,88,84}           %
\definecolor{ourdarkgray}{HTML}{362e2c}
\definecolor{ourlightgray}{HTML}{bcb1b0}

\definecolor{ourblue2}{RGB}{9,134,223} %
\definecolor{ourdarkblue2}{RGB}{5,97,164} %
\definecolor{ourlightblue2}{RGB}{132,201,250} %

\definecolor{ourorange2}{RGB}{224,90,18} %
\definecolor{ourdarkorange2}{RGB}{160,63,9} %
\definecolor{ourlightorange2}{RGB}{246,175,137} %

\definecolor{ouryellow2}{RGB}{227,213,25} %
\definecolor{ourdarkyellow2}{RGB}{177,166,17} %
\definecolor{ourlightyellow2}{RGB}{242,235,140} %

\definecolor{ourpink2}{RGB}{247,24,139} %
\definecolor{ourdarkpink2}{RGB}{164,4,86} %
\definecolor{ourlightpink2}{RGB}{250,163,207} %

\definecolor{ourgreen2}{RGB}{159,198,52} %
\definecolor{ourdarkgreen2}{RGB}{109,138,30} %
\definecolor{ourlightgreen2}{RGB}{209,228,154} %

\definecolor{ourgray2}{RGB}{124,124,115} %
\definecolor{ourdarkgray2}{RGB}{87,87,81} %
\definecolor{ourlightgray2}{RGB}{194,194,189} %

\graphicspath{{figures/}}
\newtheorem{theorem}{Theorem}
\newtheorem{lemma}[theorem]{Lemma}
\newtheorem{proposition}[theorem]{Proposition}
\newtheorem{corollary}[theorem]{Corollary}
\theoremstyle{definition}
\newtheorem{definition}[theorem]{Definition}
\newtheorem{remark}[theorem]{Remark}

\newcommand{\R}{\mathbb{R}}
\newcommand{\E}{\mathbb{E}}

\renewcommand{\P}{\mathcal{P}}
\renewcommand{\L}{\mathcal{L}}
\newcommand{\del}{\partial}
\renewcommand{\d}{\text{d}}
\newcommand{\sg}{\operatorname{sg}}

\newcommand{\vv}[1]{\mathbf{#1}}
\newcommand{\x}{\vv{x}}
\newcommand{\y}{\vv{y}}
\newcommand{\V}{\vv{V}}

\newcommand{\indep}{\perp \!\!\! \perp}
\newcommand{\idx}{\operatorname{idx}}

\newcommand{\yp}[1][]{\ifthenelse{\isempty{#1}}{{\textcolor{blue}{\vv{y}^{+}}}}{{\textcolor{blue}{\vv{y}^{+}_{\textcolor{black}{#1}}}}}}
\newcommand{\yn}[1][]{\ifthenelse{\isempty{#1}}{{\textcolor{orange}{\vv{y}^{-}}}}{{\textcolor{orange}{\vv{y}^{-}_{\textcolor{black}{#1}}}}}}

\newcommand{\Vp}[1]{{\textcolor{blue}{\vv{V}^{+}_{\textcolor{black}{#1}}}}}
\newcommand{\Vn}[1]{{\textcolor{orange}{\vv{V}^{-}_{\textcolor{black}{#1}}}}}

\newcommand{\Npos}{{N_\text{pos}}}
\newcommand{\Nneg}{{N_\text{neg}}}

\renewcommand{\varepsilon}{\mathbf{\epsilon}}

\newcommand{\Pot}{\Phi}  %

\newcommand{\Vs}{\vv{V}^\#}
\newcommand{\Vsp}[1]{{\textcolor{blue}{\vv{V}^{+\textcolor{black}{\#}}_{\textcolor{black}{#1}}}}}
\newcommand{\Vsn}[1]{{\textcolor{orange}{\vv{V}^{-\textcolor{black}{\#}}_{\textcolor{black}{#1}}}}}

\newcommand{\rbfsharp}{{$\text{log-KDE-Gauss}^\#$}}
\newcommand{\laplsharp}{{$\text{log-KDE-Laplace}^\#$}}
\newcommand{\driftlapl}{{Drifting-Laplace}}
\newcommand{\driftrbf}{{Drifting-Gauss}}
\newcommand{\naivedriftlapl}{{Naive-Laplace}}

\newcommand{\logkde}{{log-KDE}}

\hypersetup
{
  colorlinks=true,
  linkcolor=ourblue,
  urlcolor=ourviolet,
  citecolor=ourblue,
  linktoc=all,
}

\newcommand\blfootnote[1]{
    \begingroup
    \renewcommand\thefootnote{}\footnote{#1}
    \addtocounter{footnote}{-1}
    \endgroup
}

\title{Drifting Fields are not Conservative}
\author{%
  \textbf{Leonard T. Franz\textsuperscript{1,*}} \quad
  \textbf{Sebastian Hoffmann\textsuperscript{1,2,*}} \quad
  \textbf{Tim Weiland\textsuperscript{1}} \\[4pt]
  \textbf{Bernhard Schölkopf\textsuperscript{3}} \quad
  \textbf{Georg Martius\textsuperscript{1,3}} \\[8pt]
  \textsuperscript{1}Eberhard Karls Universität Tübingen \quad
  \textsuperscript{2}Max Planck Institute for Biogeochemistry \\[2pt]
  \textsuperscript{3}Max Planck Institute for Intelligent Systems \\[2pt]
  \textsuperscript{*}Equal Contribution%
}

\begin{document}
\maketitle
\blfootnote{Correspondence to \texttt{leonard-tobias.franz@uni-tuebingen.de} \& \texttt{shoffmann@bgc-jena.mpg.de}}

\begin{figure}[h]
    \centering \vspace{-1.5em}
    \includegraphics{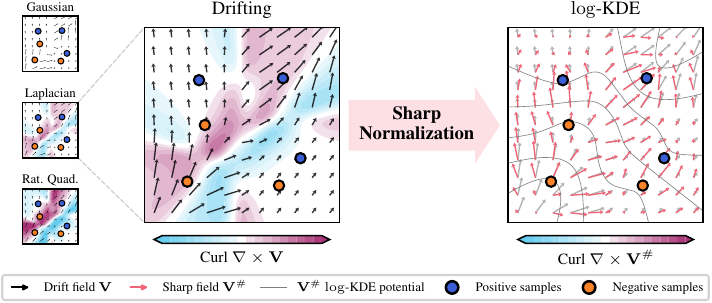}\vspace{-0.5em}
    \caption{\textbf{Sharp normalization.} Drift fields are generally not conservative, as seen by the non-vanishing curl field for the Laplacian and rational quadratic kernel (left panels). We propose a new normalization scheme via a \emph{sharp kernel} density estimate (KDE) that restores the field's conservatism (right). This allows us to reformulate the drift field as the gradient of a scalar potential, the \logkde{} potential, yielding a simplified and interpretable objective.}
    \label{fig:spotlight}
\end{figure}

\begin{abstract}
    \noindent
    Drifting models have recently gained attention for generating high-quality samples in a single forward pass.
    During training, they learn a push-forward map by following a vector-valued field, the \emph{drift field}.
    We ask whether this procedure is equivalent to optimizing a scalar loss and find that, in general, it is not: drift fields are \emph{not conservative} and cannot be written as the gradient of any scalar potential. 
    We identify the position-dependent normalization as the source of non-conservatism, with the Gaussian kernel as the unique radial exception.
    Guided by this, we introduce the \emph{sharp kernel} $k^\#$ and a sharp-normalized drift field that is conservative for general radial kernels.
    The resulting vector field is the gradient of a scalar potential that can be optimized directly using stochastic gradient descent.
    Moreover, the field has the form of a score difference of kernel density estimates, and gives exact equilibrium identifiability.
    Thus, sharp normalization closes the gap to related literature, such as Wasserstein gradient-flows and denoising score matching, also for non-Gaussian kernels.
    Empirically, sharp normalization preserves the performance of the original drifting objective, suggesting that the non-conservative flexibility is not required for high-quality generation.
\end{abstract}

\section{Introduction}
\label{sec:introduction}
Generative modeling via \emph{drifting}~\citep{deng_generative_2026} has recently attracted attention for providing high quality samples, one-step inference, and a novel training procedure, framed as fixed point iteration of generated samples. The approach formalizes training as repeatedly moving the samples~$f_{\theta}(\varepsilon)$ toward the data distribution $p$ via a so-called ``drift field''~$\V$ and then updating parameters~$\theta$ such that the generations~$f_{\theta'}(\varepsilon)$ reproduce this movement. Unlike most of machine learning, no scalar loss plays a distinguished role in the training. This  raises a natural question: is drifting implicitly optimizing a scalar loss, and if so, what is the corresponding objective?

To make our question more precise, consider the loss function employed in the drifting formalism for inducing the sample transport along the vector-valued drifting field $\V_{p,q_\theta}$:
\begin{align}
\label{eq:drifting-loss-intro}
    \L^\text{drift}(\theta) &= \E_{\varepsilon}\bigl[
        \|f_\theta(\varepsilon) - \operatorname{sg}(f_{\theta}(\varepsilon) + \V_{p,q_\theta}(f_{\theta}(\varepsilon)))\|^2
    \bigr]&\text{(``stop-gradient objective'').}
\end{align}
Here, $f_\theta:\R^n \to \R^n$ is the generator, $q_\theta$ is the push-forward distribution $(f_\theta)_\ast p_\epsilon$, and $\operatorname{sg}$ denotes the stop-gradient operator. We observe that \Cref{eq:drifting-loss-intro} can implement \emph{any} transport field $\V:\R^n \to \R^n$ for the generated outputs, including non-conservative ones, \eg with nonzero curl ($\nabla \times \V \neq 0$). In contrast, transport fields induced by optimizing a scalar field via $\nabla_\theta \Pot(f_\theta)$ are always conservative, making \Cref{eq:drifting-loss-intro} strictly more general. We ask: Is $\V$ used in the drifting framework conservative, and if not, is this additional flexibility required, or is it sufficient to follow a conservative field?

\citeauthor{deng_generative_2026} take a step toward answering this question by relating the gradient of the squared MMD loss~\citep{gretton2006kernel} to drifting field transport. They find that in the Gaussian case the MMD gradient field shares the same direction as the drift field but lacks position-dependent normalization, and that optimizing the MMD does not work sufficiently well in practice. They argue that the advantage of the drifting framework lies precisely in the stop-gradient's extra generality, as it allows them to incorporate position-dependent normalization and infer that no scalar loss can reproduce that.

In this work we show that this claim is both correct and overstated. It is correct in that \citeauthor{deng_generative_2026}'s specific normalization is not reproducible by any loss: for most kernel choices, the drift field is \emph{not conservative}---no scalar potential has it as its gradient. It is overstated because loss-based training with normalization \emph{is} possible---just with a different one. We introduce the \emph{sharp kernel} $k^\#$, which, substituted for $k$ in \citeauthor{deng_generative_2026}'s normalization, yields a conservative field whose potential---the \logkde{} potential---can serve as a loss. This potential is a KL divergence between sharp-kernel KDEs of generations and data, thus making the training objective clear and intuitive.%

Beyond recovering a scalar objective, the sharp formulation casts drifting as score-difference transport. The Wasserstein gradient flow of the KL divergence (equivalently, the Fokker--Planck equation) transports particles with velocity $\nabla\log p-\nabla\log q_t$. Our sharp field takes the same form, but being based on sharp-kernel KDEs: $\Vs_{p,q}=\nabla\log p_{\mathrm{KDE}}[k^\#]-\nabla\log q_{\mathrm{KDE}}[k^\#]$. Sharp normalization thus extends the score-difference transport interpretation of drifting to non-Gaussian radial kernels.

Our experiments confirm that sharp normalization performs comparably to the original drift field, suggesting that the non-conservative component does not contribute significantly to performance. The reformulation of drifting model training with a scalar potential allows us to simplify the implementation and removes the requirement to explicitly construct a sample transport field $\V$.

\section{Background}
\label{sec:background}

\paragraph{Kernels.}

We use the term \emph{kernel} to denote a symmetric bivariate function $k: \R^n \times \R^n \to \R$ that measures similarity between two points. Much of our analysis does not depend on positive definiteness of $k$.

\paragraph{Drifting models.}

A generator network $f_\theta: \R^n \to \R^n$ maps noise $\varepsilon \sim p_\varepsilon \in \P(\R^n)$ to samples $\x = f_\theta(\varepsilon)$, inducing a pushforward distribution $q_\theta = (f_\theta)_{\ast} p_\varepsilon$. Drifting models evolve $q_\theta$ during training via a \emph{drift field} $\V_{p,q_\theta}: \R^n \to \R^n$ that transports generated samples $\x$ toward the data distribution $p$:
\begin{equation}
\label{eq:drift-sample-update}
    \x_{i+1} = \x_i + \V_{p,q_{\theta_i}}(\x_i).
\end{equation}
\citeauthor{deng_generative_2026} construct the drift field from an attractive component $\Vp{p}$ pulling toward data samples $\yp$ and a repulsive component $\Vn{q}$ pushing away from generated samples $\yn$:
\begin{align}
\label{eq:drift-field-components}
    \begin{aligned}
        \V^\text{drift}_{p,q}[k](\x) :=
        \underbrace{
        \frac{1}{Z_p(\x)} \E_{\yp \sim p}[k(\x, \yp)(\yp - \x)]}_{
            := \Vp{p}^\text{drift}
        } -
        \underbrace{
        \frac{1}{Z_q(\x)} \E_{\yn \sim q}[k(\x, \yn)(\yn - \x)]}_{
            := \Vn{q}^\text{drift}
        }
    \end{aligned}
\end{align}
where $k$ is a kernel and $Z_p(\x) := \E_{\yp \sim p}[k(\x, \yp)]$, $Z_q(\x) := \E_{\yn \sim q}[k(\x, \yn)]$ are normalization factors. By construction, $\V_{p,q}$ is anti-symmetric ($\V_{p,q} = -\V_{q,p}$) and vanishes when $q = p$.
\citet{deng_generative_2026} formulate training as a fixed-point iteration: at step $i$, the network should satisfy $f_{\theta_{i+1}}(\varepsilon) \leftarrow f_{\theta_i}(\varepsilon) + \V_{p,q_{\theta_i}}(f_{\theta_i}(\varepsilon))$, which is implemented via the stop-gradient objective~\eqref{eq:drifting-loss-intro}. In the limit $i\to\infty$, the training likely approaches a fixed point where the field $\V$ vanishes ($\V = \mathbf{0}$). Under certain regularity conditions on $p$ one then obtains $p = q$.

\begin{definition}[Radial kernel]
\label{def:radial}
A kernel is called \emph{radial} if it depends only on the squared distance: $k(\x, \y) = \phi(\|\x - \y\|^2)$ with the radial profile $\phi:\R \to \R$.
\end{definition}

\begin{definition}[Kernel Density Estimate]
\label{def:kde}
Given a distribution $p$ and a kernel $k$, the \emph{kernel density estimate} (KDE) of $p$ is
\begin{equation}
\label{eq:kde}
    p_{\mathrm{KDE}}[k](\x) := \E_{\y \sim p}[k(\x, \y)].
\end{equation}
\end{definition}

The normalization factors in \cref{eq:drift-field-components} are exactly the KDEs: $Z_p(\x)\!=\!p_{\mathrm{KDE}}(\x)$ and $Z_q(\x)\!=\!q_{\mathrm{KDE}}(\x)$. When $k$ is positive definite, an alternative interpretation is the pointwise evaluation of the kernel mean embeddings $\mu_p=\mathbb E_{y\sim p}k(\cdot,y)$ and $\mu_q=\mathbb E_{y\sim q}k(\cdot,y)$ \cite{SonZhaSmoGreetal08}. %

\paragraph{Conservatism.}
\label{par:background-conservative}

Whether a transport field can be cast as loss-based training depends on whether it arises as the gradient of a scalar potential---a property called \emph{conservatism}.
\begin{definition}[Conservative vector field]
A vector field $\V: \R^n \to \R^n$ is \textbf{conservative} if there exists a scalar field $\Pot: \R^n \to \R$ such that $\V(\x) = -\nabla_\x \Pot(\x)$ for all $\x \in \R^n$.
\end{definition}
This is equivalent to path-independence of line integrals, $\oint_\gamma \V \cdot \d\x = 0$ for every closed curve $\gamma$; we adopt the gradient formulation as more directly useful here. A standard characterization is that:
\begin{lemma}
\label{lem:jacobian_conservative}
A vector field $V: \R^n \to \R^n$ is conservative if and only if its Jacobian is symmetric, i.e.
\begin{align*}
    \frac{\del\V_j}{\del x_i} &= \frac{\del \V_i}{\del x_j} \quad \text{for all } \quad 1 \leq i, j \leq n. \tag*{\text{\textit{Proof.} See \cref{app:conservative-proof}. \qed}}
\end{align*}
\end{lemma}%
The Jacobian's off-diagonal asymmetry quantifies the deviation from conservatism, summarized by a generalized $n$-dimensional \emph{curl} $\nabla \times \V$ (Def.~\ref{def:curl}) that agrees with the classical 2D and 3D notions.

\section{Drift Fields Are Not Generally Conservative}
\label{sec:conservative}
Plotting the curl of the drift field $\V^\text{drift}_{p,q}$ for different kernel choices reveals that, in general, the drift field is \emph{not} conservative (see \cref{fig:spotlight}, left). While the Gaussian kernel yields a curl of exactly zero, the Laplacian kernel and the rational quadratic kernel both produce nonzero curl. This means no scalar potential $\Pot$ exists whose gradient recovers these drift fields.\footnote{Note that the curl also persists when plotting only the positive or negative component of the field. We also show that a 2D counterexamples suffices constructing counter examples in arbitrary dimension (see \Cref{sec:nonconservatism-arbitrary-dim}).}

The source of the non-conservatism is the position-dependent normalization by $Z(\x)$. To see this, note that the unnormalized drifting fields $\E_p[k(\x, \y)(\y - \x)]$ coincides with the gradient of the scalar MMD objective for a particular kernel choice. We analyze this in the following section.

\subsection{Nonconservatism is Caused by the Normalization}
\label{sec:mmd-connection}

\begin{table}[t]
    \centering
    \caption{\textbf{Sharp and flat kernels for common radial kernels.} Here, $r := \|\x-\y\|$. The Laplacian is the $\nu = \tfrac{1}{2}$ boundary case of the half-integer Mat\'ern family; Mat\'ern-specific notation ($c$, $\sigma_\#$, $\sigma_\flat$, $Q_p$) and derivations are given in \Cref{app:matern-derivation}.}
    \label{tab:kernel-examples}
    \tablestyle{6pt}{1.02}
    \tablefontsize
    \begin{tabular}{l | c | c | c | c}
    \toprule
     & Gaussian & Laplacian & Mat\'ern ($\nu = p + \tfrac{1}{2}$) & Rational quadratic \\
    \midrule
    \textbf{Kernel} $k(\x, \y)$ & $\exp(-r^2/2\sigma^2)$ & $\exp(-r/\sigma)$ & $k_{\nu,\sigma}(\x,\y) := Q_p(cr)\, e^{-cr}$ & $(1 + r^2/\sigma^2)^{-2}$ \\[0.5em]
    \textbf{Sharp} $k^\#(\x,\y)$ & $\sigma^2\, k(\x,\y)$ & $\sigma(r + \sigma)\, k(\x,\y)$ & $\sigma^2\, k_{\nu+1,\,\sigma_\#}(\x,\y)$ & $\frac{\sigma^2}{2}\, k(\x,\y)^{1/2}$ \\[0.5em]
    \textbf{Flat} $k^\flat(\x,\y)$ & $\sigma^{-2}\, k(\x,\y)$ & $k(\x,\y) / (\sigma r)$ & $\frac{\nu}{(\nu-1)\sigma^2}\, k_{\nu-1,\,\sigma_\flat}(\x,\y)$ ($\nu \geq \tfrac{3}{2}$) & $\frac{4}{\sigma^2}\, k(\x,\y)^{3/2}$ \\
    \bottomrule
    \end{tabular}
\end{table}

\citet{deng_generative_2026} observe in Appendix~C.2 that the drift field is closely related to the gradient of the squared MMD loss~\citep{gretton2006kernel}, and, for the Gaussian kernel, derive a proportionality between the two. We reiterate this relation to show that the unnormalized field \emph{is} conservative, and only through the introduction of the normalization, we get nonzero curl. Adopting our notation and after some transformations (see \cref{app:mmd-derivation}) the squared MMD loss becomes
\begin{align}
    \begin{aligned}
    \L^{\text{MMD}^2}[k](\theta) &=
    \E_{\x\sim q_\theta}\left[
        \Pot_{p, \sg(q_\theta)}^\text{MMD}[k](\x)
    \right],\\
    \text{where} \qquad \Pot_{p, \sg(q_\theta)}^\text{MMD}[k](\x) &=
    \E_{\yn \sim \sg(q_\theta)}[k(\x, \yn)] -
    \E_{\yp \sim p}[k(\x, \yp)].
    \end{aligned}
\end{align}

For a radial kernel $k(\x, \y) = \phi(\|\x-\y\|^2)$, we identify the drift field $\V^\text{MMD}_{p,q}[k](\x)$ of this objective as
\begin{align}
    \begin{aligned}
    \V^\text{MMD}_{p,q}[k](\x) &:= 
        - \nabla_\x \Pot_{p, \sg(q_\theta)}^\text{MMD}[k](\x) \\
        &=
        \E_{\yp \sim p}[-2\phi'(\|\x-\yp\|^2)(\yp - \x)] -
        \E_{\yn \sim q}[-2\phi'(\|\x-\yn\|^2)(\yn - \x)].
    \end{aligned}
\end{align}
For the Gaussian kernel, where $-2 \phi'=\frac{1}{\sigma^2}\phi$, this is proportional to the unnormalized drifting field:
\begin{equation}
    \V_{p,q}^\text{unnorm.}[k](\x):=
    \E_{\yp \sim p}[k(\x, \yp)(\yp - \x)] -
    \E_{\yn \sim q}[k(\x, \yn)(\yn - \x)]
\end{equation}
 More generally if $-2\phi'(\|\x-\y\|^2)$ happens to align with another kernel, then we get its unnormalized drift field. The need to relate kernels losses like the MMD and those emerging in the corresponding drifting fields, motivates the following definition.

\paragraph{Flat and Sharp Kernels}
\label{sec:sharp-flat-def}
\begin{definition}[Sharp and flat] \label{def:sharp-flat}
The \textbf{sharp} $k^\#$ and \textbf{flat} $k^\flat$ of a kernel $k$ are kernels satisfying
\begin{equation}
\label{eq:sharp_flat}
    \underbrace{\nabla_\x k^\#(\x, \y) = k(\x, \y)(\y - \x)}_{\text{sharp}}
    \qquad
    \underbrace{\nabla_\x k(\x, \y) = k^\flat(\x, \y)(\y - \x)}_{\text{flat}}.
\end{equation}
\end{definition}
The notation is borrowed from musical accidentals: the sharp $k^\#$ raises the kernel (antiderivative), while the flat $k^\flat$ lowers it (derivative). By substituting both sides of \cref{eq:sharp_flat} into each other, it follows directly that $(k^\#)^\flat\!=\!(k^\flat)^\#\!=\!k$. Neither operation is guaranteed to exist in general nor preserve positive definiteness. The flat $k^\flat$ requires $k$ to be differentiable and may be singular where $\|\x - \y\| = 0$ (e.g., Laplacian kernel). %
For a radial kernel $k(x,y) = \phi(\|\x - \y\|^2)$, both admit closed forms:\looseness-1
\begin{equation}
\label{eq:sharp-flat-radial}
    k^\#(\x, \y) = \frac{1}{2}\int_{\|\x-\y\|^2}^\infty \phi(r)\, \d r, \qquad
    k^\flat(\x, \y) = -2\phi'(\|\x - \y\|^2).
\end{equation}

\begin{figure}[t]
\centering
\includegraphics{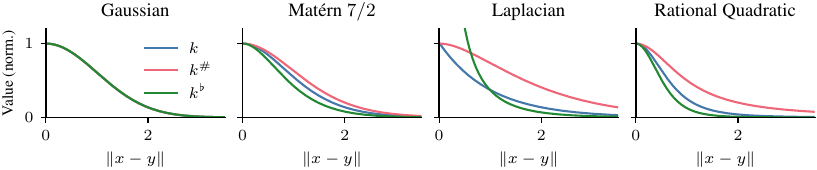}
\caption{\textbf{Normalized radial profiles of flat and sharp kernels}. For the Laplacian, Mat\'ern-$7/2$, and rational quadratic kernels, the sharp and flat profiles have distinct shapes, while for the Gaussian all agree up to scale.}
\label{fig:sharp-flat-curves}
\end{figure}

The flat-sharp formalism clarifies the relationship between MMD gradients and drifting fields:
\begin{proposition}
\label{lem:conservative}
Let $k(\x,\y)$ be a kernel for which $k^\#(\x, \y)$ exists, then 
\begin{equation}
    \V_{p,q}^\textnormal{unnorm.}[k](\x) = -\nabla_\x \Pot_{p, \sg(q_\theta)}^\textnormal{MMD}[k^\#](\x).
\end{equation}
\end{proposition}
\begin{proof}
\begin{align}
    \begin{aligned}
    -\nabla_\x \Pot_{p, \sg(q_\theta)}^\text{MMD}[k^\#](\x) &= \E_{\yp \sim p}[\nabla_\x k^\#(\x, \yp)] - \E_{\yn \sim \sg(q_\theta)}[\nabla_\x k^\#(\x, \yn)] \\
    &= \underbrace{\E_{\yp \sim p}[k(\x, \yp)(\yp - \x)]}_
    {:=\Vp{p}^\text{MMD}}
    - \underbrace{\E_{\yn \sim q}[k(\x, \yn)(\yn - \x)]}_
    {:=\Vn{q}^\text{MMD}} \\
    &= \V_{p,q}^\text{unnorm.}[k](\x) .
    \end{aligned}
\end{align}%
\end{proof}

It is thus conservative and has zero curl.
This shows us that the non-conservatism is a byproduct of the normalization. %

\paragraph{Radiality is necessary.}
\label{sec:radial-kernels}

We have shown that if $k$ is assumed to be radial, we find a potential given by the squared MMD whose gradient produces the unnormalized drifting field, which is thus, by definition, conservative. The reverse is also true. We refer the reader to \cref{lem:radial-kernel} for the proof.

\subsection{The Gaussian Kernel: An Exception where Normalization Retains Conservatism}
\label{sec:gaussian-potential}
The exception where the normalization does not lead to non-conservatism is the Gaussian kernel (see \Cref{fig:spotlight}). We are thus able to identify the exact scalar potential whose gradient produces the field:
\begin{proposition}
\label{prop:gaussian-kde}
For the Gaussian kernel with bandwidth $\sigma$, the drift sub-fields satisfy
\begin{equation}
    \Vp{p}(\x) = \sigma^2 \nabla_\x \log p_{\mathrm{KDE}}[k](\x), \qquad
    \Vn{q}(\x) = \sigma^2 \nabla_x \log q_{\mathrm{KDE}}[k](\x).
\end{equation}
\end{proposition}
\begin{proof}
We compute $\nabla_\x \log q_{\mathrm{KDE}}[k](\x) = \frac{\nabla_\x q_{\mathrm{KDE}}[k](\x)}{q_{\mathrm{KDE}}[k](\x)}$. Using \cref{def:kde} and that for Gaussian kernels, we have $\nabla_\x k(\x, \y) = \sigma^{-2} k(\x, \y)(\y - \x)$ (See Corollary~\ref{corr:gaussian_kernel_sharp_proof} and Definition \ref{def:sharp-flat}), we obtain:
\begin{equation*}
    \nabla_\x \log q_{\mathrm{KDE}}[k](\x) = \frac{\E_q[\nabla_\x k(\x, \y)]}{\E_q[k(\x, \y)]}
    = \frac{1}{\sigma^2} \frac{\E_q[k(\x, \y)(\y - \x)]}{Z_q(\x)}
    = \frac{1}{\sigma^2} \Vn{q}(\x).
\end{equation*}
The same argument applies analogously to $\Vp{p}(\x)$.
\end{proof}

The composite drift field defined in \cref{eq:drift-field-components} thus reduces to 
\begin{equation}
    \V_{p,q}(\x)
    = \sigma^2(\nabla_\x \log p_{\mathrm{KDE}}[k](\x)
    - \nabla_\x \log q_{\mathrm{KDE}}[k](\x))
    = -\sigma^2 \nabla_\x \log\left(\frac{q_{\mathrm{KDE}}[k](\x)}{p_{\mathrm{KDE}}[k](\x)}\right).
\end{equation}

\section{Sharp Normalization}
\label{sec:sharp-flat}
Having found the normalization of the drifting field to be the culprit for nonconservatism, and that, in some cases, normalization does \emph{not} destroy conservatism we propose a new normalization which restores conservatism for \emph{any} radial kernel. 
This normalization is given by
\begin{equation}
\label{eq:sharp-kde}
    Z_p^\#(\x) := \E_{\yp \sim p}[k^\#(\x, \yp)], \qquad Z_q^\#(x) := \E_{\yn \sim q}[k^\#(\x, \yn)] 
\end{equation}
for the positive and negative fields. Note that $Z_p^\# = p_\mathrm{KDE}[k^\#]$ and $Z_q^\# = q_\mathrm{KDE}[k^\#]$.
By the defining property of the sharp kernel, $\nabla_x k^\#(x,y) = k(x,y)(y-x)$, the drift numerator is the gradient of the sharp KDE. Dividing by the corresponding sharp KDE therefore recovers a log-KDE score:
\begin{align}
\label{eq:sharp-field}
    \begin{aligned}
        \Vsp{p}(\x) &:= \frac{1}{Z_p^\#(\x)} \E_{\yp \sim p}[k(\x, \yp)(\yp - \x)] = \nabla_\x \log p_\mathrm{KDE}[k^\#](\x), \\
        \Vsn{q}(\x) &:= \frac{1}{Z_q^\#(\x)} \E_{\yn \sim q}[k(\x, \yn)(\yn - \x)] = \nabla_\x \log q_\mathrm{KDE}[k^\#](\x).
    \end{aligned}
\end{align}
\begin{definition}[Sharp-normalized drift field] We define the sharp-normalized drift field as 
\begin{align}
\Vs_{p,q}(\x) &:= \Vsp{p}(\x) - \Vsn{q}(\x)
= -\nabla_\x \log \left(
        \frac
        {q_{\mathrm{KDE}}[k^\#](\x)}
        {p_{\mathrm{KDE}}[k^\#](\x)}
        \right)
\end{align}
\end{definition}

\paragraph{KDE-smoothed score transport.}
The identity above reveals that sharp-normalized drifting is a score-difference field for the smoothed densities $p^\#:=p_{\mathrm{KDE}}[k^\#]$ and $q^\#:=q_{\mathrm{KDE}}[k^\#]$:
\begin{equation}
\label{eq:sharp-score-difference}
    \Vs_{p,q}(\x)
    =
    \nabla_\x\log p^\#(\x)-\nabla_\x\log q^\#(\x).
\end{equation}
This is the KDE analogue of the velocity $\nabla\log p-\nabla\log q_t$ appearing in KL Wasserstein gradient flow. When $k^\#$ is nonnegative and integrable, it can be normalized as a corruption kernel, so $\nabla\log p^\#$ is also a denoising score: the posterior average of conditional kernel scores. The Gaussian case is special because $k^\#\propto k$, so the original drifting normalization already uses the correct smoothed density. For non-Gaussian kernels, sharp normalization restores this score-flow structure.

\paragraph{Scalar potential training.}
Since $\Vs_{p,q}(\x)$ is conservative, the stop-gradient training objective simplifies. In general, when $\V_{p,q}(\x) = -\nabla_\x \Pot_{p, q}(\x)$ for some scalar potential $\Pot_{p, q}: \R^n \to \R$, we show (see derivation in Appendix~\ref{app:conservative_training}) that
\begin{align}\label{eq:loss_identity}
    \begin{aligned}
        \nabla_\theta \E_{\varepsilon}\left[\L^\text{drift}(\theta)\right]
        = \;2\nabla_\theta \E_{\x\sim q_\theta}
          \left[\Pot_{p, \sg(q_\theta)}(\x)\right].
    \end{aligned}
\end{align}
That means training reduces, up to a positive scalar factor, to performing a gradient descent step with the
scalar field $\E_{\x\sim q_\theta}[\Pot_{p, \sg(q_\theta)}(\x)]$ at each step, without explicitly
constructing the transport field $\V$.
This moves away from the $V$-centric framework of \citet{deng_generative_2026}, replacing the explicit construction and application of a transport field with a single scalar loss that can be implemented and optimized straightforwardly with standard automatic differentiation.

Applying this to the sharp-normalized field yields the \logkde{} loss:
\begin{equation}
    \label{eq:sharp-loss}
    \boxed{
        \L^{\text{\logkde}}_{p,q}(\theta) := \E_{\x \sim q_\theta}\left[\log\left(
        \frac
        {q_{\mathrm{KDE}}[k^\#](\x)}
        {p_{\mathrm{KDE}}[k^\#](\x)}
        \right)\right]
    }.
\end{equation}
Importantly, $q_{\mathrm{KDE}}[k^\#](\x) = \E_{\yn \sim \sg(q_\theta)}\left[k^\#(\yn, \x)\right]$ is expected to only depend on $\theta$ through $\x$ and not via $\yn$ here.
Note that the potential $\Pot$ depends on the current generated distribution $q_\theta$ through the stop-gradient operation $\sg$.
This means the loss landscape shifts between training steps as $q_\theta$ evolves, and we are not descending a single fixed objective landscape, but a smoothly changing sequence of landscapes, inherently different from, for example, the MMD loss where a truly fixed potential exists.
Note that the loss is structurally similar to the KL divergence, with the exception of the sampling distribution $q_\theta$. As shown in \Cref{alg:our_loss}, for exponential-style kernels, the above objective can be implemented in a numerically stable way using the \texttt{logsumexp} operation. %

\begin{proposition}[Identifiability of sharp-normalized equilibria]
\label{prop:sharp-identifiability}
Assume $k^\#$ is strictly positive, continuously differentiable in its first argument, characteristic \cite{SimSch18}, and has constant finite mass $\int_{\R^n} k^\#(\x,\y)\,\d \x = c_\# \in (0,\infty)$ for every $\y$. If the domain is connected, then $\Vs_{p,q}(\x)=0$ for all $\x$ if and only if $p=q$. %
\end{proposition}

\emph{Proof sketch.}
By \Cref{eq:sharp-field}, $\Vs_{p,q}=0$ implies that the log-ratio $q_{\mathrm{KDE}}[k^\#]/p_{\mathrm{KDE}}[k^\#]$ has zero gradient. On a connected domain, this ratio must be constant. Since $k^\#$ has the same finite mass $c_\#$ around every point, both KDEs integrate to $c_\#$, which fixes the constant to one. Therefore the two sharp KDE embeddings agree pointwise, and characteristicness of $k^\#$ gives $p=q$. The converse direction is immediate by construction, as already observed for drifting fields by \citet{deng_generative_2026}. Hence sharp-normalized drifting has exact equilibrium identifiability. The full proof is given in \Cref{app:sharp-identifiability}. Compared to the companion-elliptic proof of \citet{lee_identifiability_2026} for the original Deng field, this argument is much simpler because sharp normalization uses the same companion kernel in the numerator and denominator. In the companion-elliptic Mat\'ern class of \citet{lee_identifiability_2026}, their companion kernel is exactly our sharp kernel up to scale.

\begin{algorithm}[t]
\caption{\logkde{} Loss (Gaussian)}
\label{alg:our_loss}
\begin{lstlisting}[keepspaces=true]
def log_KDE_loss(x: "[N, D]", y_pos: "[N_pos, D]", y_neg: "[N_neg, D]", T: float):
    dist_pos, dist_neg = cdist(x, y_pos), cdist(x, y_neg.detach())  # compute pairwise distances
    logits_pos, logits_neg = -dist_pos**2 / (2 * T**2), -dist_neg**2 / (2 * T**2)  # compute logits
    logits_neg.fill_diagonal_(-inf)  # exclude diagonal
    log_p = logsumexp(logits_pos, dim=1) - log(N_pos)  # compute log KDEs
    log_q = logsumexp(logits_neg, dim=1) - log(N_neg - 1)
    return (log_q - log_p).mean()  # compute loss
\end{lstlisting}
\end{algorithm}

\paragraph{When does the sharp-normalized field coincide with the drift field?}
\label{sec:fields-coincide}

The sharp-normalized field and the drift field coincide (up to a global constant) if and only if $p_{\mathrm{KDE}}[k^\#](\x) / p_{\mathrm{KDE}}[k](\x)$ is independent of $\x$, which holds when $k^\# \propto k$. For radial kernels $k = \phi(\|\x-\y\|^2)$, this becomes:
\begin{equation}
    \frac{1}{2}\int_{\|\x-\y\|^2}^\infty \phi(r)\, \d r \propto \phi(\|\x-\y\|^2),
\end{equation}
whose unique solution (up to scaling) is $\phi(r) = \exp(-r / 2\sigma^2)$: the Gaussian kernel.

\begin{proposition}
\label{prop:gaussian-unique}
Among radial kernels, the Gaussian is the unique kernel for which the drift field and the sharp-normalized field coincide. It is characterized by the property $k^\# \propto k$.
\end{proposition}

\begin{remark}[Half-integer Mat\'ern ladder]
\Cref{prop:gaussian-unique} sits naturally as the $\nu \to \infty$ endpoint of a smoothness ladder on the half-integer Mat\'ern family: sharp and flat shift $\nu \to \nu \pm 1$ (\Cref{cor:matern-half-shift}), with the Laplacian ($\nu = \tfrac{1}{2}$) as the lower boundary where the flat side becomes singular (\Cref{corr:laplace_flat_sharp_proof}). The Mat\'ern family thus interpolates between the Laplacian and Gaussian cases.
\end{remark}

\subsection{Comparing the Original and Sharp-Normalized Fields}
Having established that sharp normalization restores conservatism, we now characterize how it differs from the original drift field and relate both to MMD-based training.
 The vanilla $\V_{p,q}[k]$, MMD-induced $\V^\text{MMD}_{p,q}[k^\#]$, and \logkde{}-induced field $\V^\#_{p,q}$ share the same direction-inducing terms $\E_{\yp \sim p}[k(\x, \yp)(\yp - \x)]$ and $\E_{\yn \sim q}[k(\x, \yn)(\yn - \x)]$, differing only in their normalizations. Because the positive and negative components are rescaled independently, the composite fields generally differ in both magnitude \emph{and} direction (i.e., $\V^\text{MMD}_{p,q}[k^\#] \not\propto \V_{p,q}[k] \not\propto \V^\#_{p,q}$). \Cref{fig:normalization-ratio} shows the discrepancy in magnitude is most pronounced far from the data.

\begin{figure}[t]
    \centering
    \includegraphics{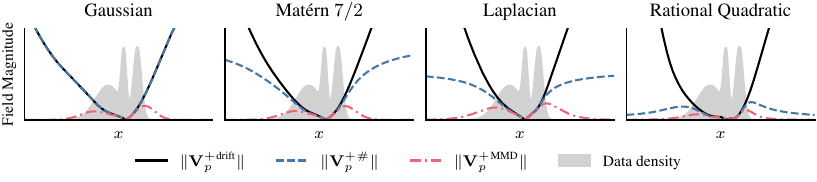}
    \caption{\textbf{Magnitudes of positive drift fields.} Given a fixed target distribution (gray shading), we show the magnitude of the original drift field~(black), sharp-normalized drift field~(dashed), and the MMD gradient~(dash-dotted) as function of $x$. The MMD gradient decays fast in the tails while the drift field grows unboundedly. The sharp-normalized field lies between these extremes, with Mat\'ern-$7/2$ showing intermediate behavior between the Laplacian and Gaussian/RQ panels.}
    \label{fig:normalization-ratio}
\end{figure}

 \paragraph{Interpretation.} The three fields realize three levels of normalization:
\textbf{1.)} \emph{None.} $\V_{p,q}^\text{unnorm.}[k]$ is the gradient of $\Pot_{p, \sg(q_\theta)}^\textnormal{MMD}[k^\#]$---hence conservative---but vanishes far from data.
\textbf{2.)} \emph{$Z_p, Z_q$.} The drift field $\V_{p,q}[k]$ of \citet{deng_generative_2026} fixes the vanishing-signal problem but is not, in general, a gradient.
 \textbf{3.)} \emph{Sharp ($Z_p^\#, Z_q^\#$).} $\V^\#_{p,q}$ is the gradient of the \logkde{} loss~\eqref{eq:sharp-loss} (hence conservative) and shows improved tail behavior over the MMD gradient (e.g., bounded signal for the Laplacian kernel).

\section{Experiments}
\label{sec:experiments}
We conduct two sets of experiments on generative image modeling to evaluate our proposed \logkde{} loss. First, we assess pixel-space image generation on MNIST~\citep{lecun2002gradient} and Fashion-MNIST~\citep{xiao2017fashion} in a controlled setting, deliberately omitting several enhancements introduced by \citet{deng_generative_2026} that are equally applicable to any similarly structured objective. Second, we evaluate the \logkde{} objective on the more challenging task of latent-space image generation on ImageNet~\citep{imagenet2009}, using the original codebase of \citeauthor{deng_generative_2026}\footnote{\href{https://github.com/lambertae/drifting}{github.com/lambertae/drifting}} Full configuration details for all experiments are provided in \cref{tab:exp_configs}.

\paragraph{Pixel-space generation on MNIST and Fashion-MNIST.}
Following \citet{deng_generative_2026}, we train DiT-S/2~\citep{peebles2023scalable} models with register tokens and style-embedding tokens to directly map noise to pixel images using AdamW~\citep{loshchilov2017decoupled}. Models are trained for 12{,}000 steps on MNIST and 16{,}000 steps on the more challenging Fashion-MNIST, with all loss computations performed in pixel space. Unlike \citeauthor{deng_generative_2026}, we do not normalize features and distances per batch. Instead, we maintain, per class, an exponential moving average~($\beta{=}0.99$) of the mean $\ell_2$-distance from each data sample to its $k$-nearest neighbors~($k{=}5$) within a batch, and divide each true and generated sample by this scalar to achieve scale-invariant kernel widths. Furthermore, we employ a single kernel width, omit classifier-free guidance~(CFG), and omit the drifting-field normalization, as there is no need to balance multiple objectives. All hyperparameters are fixed across experiments; only the training objective varies.

\paragraph{Latent-space generation on ImageNet.}
Generating high-fidelity images on ImageNet requires an additional feature extractor, whose specific design can substantially influence performance~\citep[cf.][Table~3]{deng_generative_2026}. To minimize such confounding factors, we integrate the \logkde{} objective directly into the codebase of \citeauthor{deng_generative_2026}, replacing only the loss function while keeping all other components unchanged. We adopt the \emph{ablation default} configuration, training a DiT-B/2 model to generate outputs in the latent space of a pretrained SD-VAE~\citep{rombach2022highres, stabilityai/sd-vae-ft-mse} and computing the loss in the latent space of a pretrained latent-MAE-256 model~\citep{deng_generative_2026}. We retain the same batch-wise feature normalization strategy. %
We incorporate CFG into the \logkde{} objective in \cref{alg:our_loss} via a weighted \emph{logsumexp} expression to evaluate $\log_q(\x) = \log \E_{\yn \sim \sg(\hat{q}_\theta)}\left[k^\#(\yn, \x)\right]$, where $\hat{q}_\theta$ is a mixture of $q_\theta$ and the unconditional data distribution $p(\cdot | \varnothing)$. We find that CFG is essential for achieving competitive results.\looseness-1

\paragraph{Corrected drifting objective.}
We observe that the implementation of \citeauthor{deng_generative_2026} does not compute the drifting field as defined in \cref{eq:drift-field-components}. Instead, it computes
\begin{align}
\bar{\V}_{p,q}(\x) = \frac{1}{(Z_p(\x) + Z_q(\x))^2}\E_{p,q}\left[k(\x,\yp)k(\x,\yn)(\yp - \yn)\right].
\end{align}
We refer to \cref{app:wrong_drifting} for a detailed derivation. In our pixel-space experiments, we therefore compare three objectives: the drifting objective actually used in \citeauthor{deng_generative_2026}, a corrected variant that evaluates the field exactly as in \cref{eq:drift-field-components}, and the \logkde{} objective. We denote these as \emph{Drifting-X}, \emph{Naive-X}~(\cref{alg:correct_drifting}), and \emph{\logkde-X}~(\cref{alg:our_loss}), respectively, where \emph{X} identifies the kernel.

\subsection{Results}

\begin{figure}[t]
    \centering
    \includegraphics{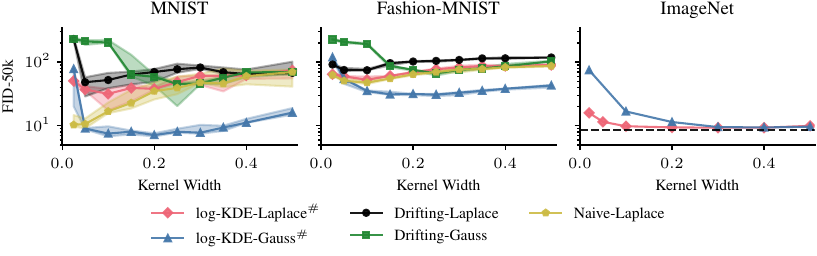}\vspace{-.5em}
    \caption{\textbf{Kernel width ablation.} The Fréchet Inception Distance (FID-50k, lower is better) as a function of kernel width for MNIST (left), Fashion-MNIST (center), and ImageNet (right). Solid lines indicate median, shaded areas indicate minimum and maximum over three seeds (ImageNet: one seed). The black dashed line indicates the best FID~(8.67) on ImageNet reported by \citeauthor{deng_generative_2026} using a single scale.} %
    \label{fig:fid}
\end{figure}

\Cref{fig:fid} reports the Fréchet Inception Distance~\citep{heusel2017gans} evaluated over $50{,}000$ samples (FID-50k) as a function of kernel width; \Cref{tab:quant_results} gives the best FID and corresponding F1 per objective, with precision and recall following \citet{kynkaanniemi2019improved}. Uncurated samples are in \cref{app:additional_results}.

Across all three datasets, the scalar \logkde{} loss matches or beats the drifting field in FID. The gap is largest on MNIST, where \rbfsharp{} and \laplsharp{} reach 7.21 and 29.12 versus 37.62 and 45.02 for their drifting counterparts. The corrected \emph{Naive} implementation also generalizes to smaller kernel widths and consistently outperforms the original implementation of \citeauthor{deng_generative_2026}

On ImageNet, both \logkde{} variants perform similarly, with $\text{Laplace}^\#$ tolerating smaller widths. Both require larger bandwidths~($\sigma \in [0.2, 0.4]$) than the drifting objective~($\sigma \in [0.05, 0.2]$); whether this is intrinsic to \logkde{} or due to numerical instability at small widths remains unresolved.

Taken together, these results support our central claim: the scalar \logkde{} objective is a simple yet effective alternative to drifting, and the non-conservative expressiveness of the drifting field is not required for high-quality generation.

\begin{table}[t]
    \centering
    \caption{\textbf{Best FID-50k ($\downarrow$) and F1 ($\uparrow$).} For each objective, we report the FID and F1 scores (mean and standard deviation over 3 seeds, ImageNet: one seed) for the best kernel bandwidth $\sigma$. For \emph{\driftlapl} on ImageNet, we report the FID as stated in \citeauthor{deng_generative_2026}}
    \label{tab:quant_results}
    \tablestyle{6pt}{1.02}
    \adjustbox{max width=\linewidth}{
    \begin{tabular}{@{}c@{\ \ \ }l | c c@{\ \ \ }c | c c@{\ \ \ }c | c c@{\ \ \ }c@{}}
    \toprule
              & & \multicolumn{3}{c}{MNIST} & \multicolumn{3}{c}{Fashion-MNIST} & \multicolumn{3}{c@{}}{ImageNet (ablation)} \\
    \midrule
    Kernel & Objective & FID $\downarrow$ & F1 $\uparrow$ & $\sigma$ & FID $\downarrow$ & F1 $\uparrow$ & $\sigma$ & FID $\downarrow$ & F1 $\uparrow$ & $\sigma$ \\
    \midrule
    \multirow{2}{*}{\textbf{Gauss}}
    & \rbfsharp{}       & \textbf{7.21}{\tiny\,$\pm$\,0.42} & \textbf{0.56}{\tiny\,$\pm$\,0.01} & 0.2 & \textbf{30.57}{\tiny\,$\pm$\,1.20} & \textbf{0.21}{\tiny\,$\pm$\,0.01} & 0.25 & 9.43 & 0.53 & 0.4 \\
    & \driftrbf           & 37.62{\tiny\,$\pm$\,14.97} & 0.11{\tiny\,$\pm$\,0.15} & 0.25 & 64.87{\tiny\,$\pm$\,3.03} & 0.05{\tiny\,$\pm$\,0.01} & 0.25 & $-$ & $-$ & $-$ \\
    \midrule
    \multirow{3}{*}{\textbf{Laplace}}
    & \laplsharp          & 29.12{\tiny\,$\pm$\,10.57} & 0.18{\tiny\,$\pm$\,0.16} & 0.1 & 54.57{\tiny\,$\pm$\,4.33} & 0.07{\tiny\,$\pm$\,0.02} & 0.1 & 9.35 & 0.50 & 0.4 \\
    & \naivedriftlapl     & 10.95{\tiny\,$\pm$\,1.97} & 0.51{\tiny\,$\pm$\,0.03} & 0.05 & 48.65{\tiny\,$\pm$\,4.00} & 0.12{\tiny\,$\pm$\,0.02} & 0.1 & $-$ & $-$ & $-$ \\
    & \driftlapl~\citep{deng_generative_2026} & 45.02{\tiny\,$\pm$\,14.75} & 0.05{\tiny\,$\pm$\,0.08} & 0.05 & 75.69{\tiny\,$\pm$\,2.25} & 0.02{\tiny\,$\pm$\,0.00} & 0.1 & \textbf{8.67} & n/a & 0.05 \\
    \bottomrule
    \end{tabular}
    }
    
\end{table}

\section{Related Work}
\label{sec:related_work}
\paragraph{Gradient flows and score transport.}
The variational view of probability transport as Wasserstein gradient flow goes back to the JKO formulation of the Fokker--Planck equation as KL descent in Wasserstein space~\citep{jordan_variational_1998}. In this view, particles move with score-difference velocity $\nabla\log p-\nabla\log q_t$, the same object underlying score-based generative modeling and denoising score matching~\citep{vincent_connection_2011}. Kernelized particle flows have also been studied for integral probability metrics, notably MMD gradient flows~\citep{gretton2006kernel, arbel_mmd_2019}. Our sharp formulation connects drifting to this family by showing that the drifting numerator is the gradient of a sharp KDE. %

\paragraph{Reinterpretations of drifting.}
Concurrent work connects Gaussian drifting to Wasserstein gradient flows of KDE-approximated divergences~\citep{cao_gradient_2026}, smoothed score matching via Tweedie's formula~\citep{lai_unified_2026}, and Fourier-space training dynamics~\citep{turan_generative_2026}.
Our flat-sharp formalism explains why these reinterpretations all hinge on the Gaussian: it is the unique radial kernel satisfying $k^\#\propto k$, so its drifting field already coincides with the score difference between KDEs. For other kernels this coincidence breaks down, and sharp normalization is precisely what restores it.

\paragraph{Companion kernels and identifiability.}
\citet{lee_identifiability_2026} study identifiability and stability of the original drifting field for companion-elliptic kernel families, including Mat\'ern kernels. In this class, their companion kernel coincides with our sharp kernel up to scale. Our sharp-normalized formulation makes the companion role explicit and yields a simple equilibrium-identifiability argument, since the same kernel appears in both numerator and denominator of the KDE log-ratio. Identifiability for a different normalization scheme—a two-sided Sinkhorn variant—is established by \citet{he_sinkhorn_2026}.

\section{Conclusion}
\label{sec:conclusion}
We have shown that drift fields are not generally conservative: the position-dependent normalization produces transversal magnitude variations for all radial kernels except the Gaussian. 
The obstruction is a denominator mismatch.
The drifting numerator is the gradient of a KDE built from the sharp kernel $k^\#$, but the original normalization divides by the KDE of the unsharp kernel $k$.
Sharp normalization fixes this mismatch by dividing by the corresponding sharp KDE, thereby completing the KDE score and restoring the score-difference  form $\nabla\log p_{\mathrm{KDE}}[k^\#]-\nabla\log q_{\mathrm{KDE}}[k^\#]$. 
This turns drifting into a conservative KDE-smoothed score transport with scalar \logkde{} objective. It also provides exact equilibrium identifiability and extends the Gaussian score-flow interpretation of drifting to non-Gaussian radial kernels.
Our experiments confirm that sharp normalization performs comparably to the original drifting objective, suggesting that the non-conservative component introduced by the original denominator is not required for high-quality generation.

\paragraph{Limitations and future work.}
Our population-level theory does not directly imply convergence for practical training with minibatch KDE estimates, finite-batch optimizers, neural generator parameterizations, feature-space losses, and classifier-free guidance---a limitation shared by related work on drifting models. The cleanest theory also assumes radial kernels with suitable sharp companions; the denoising interpretation requires positive, integrable $k^\#$, and performance remains bandwidth-sensitive. More broadly, the score-flow view suggests that drifting may benefit from ideas developed for Wasserstein gradient flows, denoising score matching, and score-based generative modeling, including improved objectives, schedules, and algorithms.

\FloatBarrier

{
\bibliographystyle{icml2026_titleurl}
\bibliography{references}

@article{deng_generative_2026,
      title={Generative Modeling via Drifting}, 
      author={Mingyang Deng and He Li and Tianhong Li and Yilun Du and Kaiming He},
      journal={arXiv preprint arXiv:2602.04770},
      year={2026},
      eprint={2602.04770},
      archivePrefix={arXiv},
      primaryClass={cs.LG},
      url={https://arxiv.org/abs/2602.04770}, 
}

@inproceedings{gretton2006kernel,
 author = {Gretton, Arthur and Borgwardt, Karsten and Rasch, Malte and Sch\"{o}lkopf, Bernhard and Smola, Alex},
 booktitle = {Advances in Neural Information Processing Systems},
 pages = {},
 title = {A Kernel Method for the Two-Sample-Problem},
 url = {https://proceedings.neurips.cc/paper_files/paper/2006/file/e9fb2eda3d9c55a0d89c98d6c54b5b3e-Paper.pdf},
 volume = {19},
 year = {2006}
}

@article{jordan_variational_1998,
  author = {Jordan, Richard and Kinderlehrer, David and Otto, Felix},
  title = {The Variational Formulation of the Fokker-Planck Equation},
  journal = {SIAM Journal on Mathematical Analysis},
  year = {1998},
  volume = {29},
  number = {1},
  pages = {1--17},
  doi = {10.1137/S0036141096303359},
  url = {https://doi.org/10.1137/S0036141096303359},
}

@article{vincent_connection_2011,
  author = {Vincent, Pascal},
  title = {A Connection Between Score Matching and Denoising Autoencoders},
  journal = {Neural Computation},
  year = {2011},
  volume = {23},
  number = {7},
  pages = {1661--1674},
  doi = {10.1162/NECO_a_00142},
  url = {https://doi.org/10.1162/NECO_a_00142},
}

@inproceedings{arbel_mmd_2019,
  author = {Arbel, Michael and Korba, Anna and Salim, Adil and Gretton, Arthur},
  title = {Maximum Mean Discrepancy Gradient Flow},
  booktitle = {Advances in Neural Information Processing Systems},
  year = {2019},
  volume = {32},
  url = {https://proceedings.neurips.cc/paper/2019/hash/944a5ae3483ed5c1e10bbccb7942a279-Abstract.html},
}

@article{xiao2017fashion,
      title={{Fashion-MNIST}: a Novel Image Dataset for Benchmarking Machine Learning Algorithms}, 
      author={Han Xiao and Kashif Rasul and Roland Vollgraf},
      journal={arXiv preprint arXiv:1708.07747},
      year={2017},
      eprint={1708.07747},
      archivePrefix={arXiv},
      primaryClass={cs.LG},
      url={https://arxiv.org/abs/1708.07747}, 
}

@article{lecun2002gradient,
  author={Lecun, Y. and Bottou, L. and Bengio, Y. and Haffner, P.},
  journal={Proceedings of the IEEE}, 
  title={Gradient-based learning applied to document recognition}, 
  year={1998},
  volume={86},
  number={11},
  pages={2278-2324},
  doi={10.1109/5.726791},
  url={https://ieeexplore.ieee.org/document/726791}
}

@inproceedings {peebles2023scalable,
author = { Peebles, William and Xie, Saining },
booktitle = { 2023 IEEE/CVF International Conference on Computer Vision (ICCV) },
title = {Scalable Diffusion Models with Transformers},
year = {2023},
volume = {},
ISSN = {},
pages = {4172-4182},
doi = {10.1109/ICCV51070.2023.00387},
url = {https://doi.ieeecomputersociety.org/10.1109/ICCV51070.2023.00387},
}

@inproceedings{
loshchilov2017decoupled,
title={Decoupled Weight Decay Regularization},
author={Ilya Loshchilov and Frank Hutter},
booktitle={International Conference on Learning Representations},
year={2019},
url={https://openreview.net/forum?id=Bkg6RiCqY7},
}

@inproceedings{kynkaanniemi2019improved,
 author = {Kynk\"{a}\"{a}nniemi, Tuomas and Karras, Tero and Laine, Samuli and Lehtinen, Jaakko and Aila, Timo},
 booktitle = {Advances in Neural Information Processing Systems},
 pages = {},
 title = {Improved Precision and Recall Metric for Assessing Generative Models},
 url = {https://proceedings.neurips.cc/paper_files/paper/2019/file/0234c510bc6d908b28c70ff313743079-Paper.pdf},
 volume = {32},
 year = {2019}
}

@inproceedings{heusel2017gans,
 author = {Heusel, Martin and Ramsauer, Hubert and Unterthiner, Thomas and Nessler, Bernhard and Hochreiter, Sepp},
 booktitle = {Advances in Neural Information Processing Systems},
 pages = {},
 title = {{GANs} Trained by a Two Time-Scale Update Rule Converge to a Local Nash Equilibrium},
 url = {https://proceedings.neurips.cc/paper_files/paper/2017/file/8a1d694707eb0fefe65871369074926d-Paper.pdf},
 volume = {30},
 year = {2017}
}

@inproceedings{dhariwal2021diffusion,
 author = {Dhariwal, Prafulla and Nichol, Alexander},
 booktitle = {Advances in Neural Information Processing Systems},
 pages = {8780--8794},
 title = {Diffusion Models Beat GANs on Image Synthesis},
 url = {https://proceedings.neurips.cc/paper_files/paper/2021/file/49ad23d1ec9fa4bd8d77d02681df5cfa-Paper.pdf},
 volume = {34},
 year = {2021}
}

@article{stabilityai/sd-vae-ft-mse,
  url={https://huggingface.co/stabilityai/sd-vae-ft-mse},
  title={sd-vae-ft-mse},
  journal={Hugging Face},
  editor={{Hugging Face}},
  author={{Stability AI}},
  year={2022},
}

@article{rombach2022highres,
      title={High-Resolution Image Synthesis with Latent Diffusion Models}, 
      author={Robin Rombach and Andreas Blattmann and Dominik Lorenz and Patrick Esser and Björn Ommer},
      year={2022},
      eprint={2112.10752},
      archivePrefix={arXiv},
      primaryClass={cs.CV},
      url={https://arxiv.org/abs/2112.10752},
      journal={arXiv preprint arXiv:2112.10752}
}

@INPROCEEDINGS{imagenet2009,
  author={Deng, Jia and Dong, Wei and Socher, Richard and Li, Li-Jia and Kai Li and Li Fei-Fei},
  booktitle={2009 IEEE Conference on Computer Vision and Pattern Recognition}, 
  title={{ImageNet}: A large-scale hierarchical image database}, 
  year={2009},
  volume={},
  number={},
  pages={248-255},
  doi={10.1109/CVPR.2009.5206848},
  url={https://ieeexplore.ieee.org/document/5206848}
}

@article{cao_gradient_2026,
    title = {Gradient Flow Drifting: Generative Modeling via Wasserstein Gradient Flows of {KDE}-Approximated Divergences},
    url = {http://arxiv.org/abs/2603.10592},
    shorttitle = {Gradient Flow Drifting},
    journal = {arXiv preprint arXiv:2603.10592},
    publisher = {{arXiv}},
    author = {Cao, Jiarui and Wei, Zixuan and Liu, Yuxin},
    urldate = {2026-03-20},
    year = {2026},
    date = {2026-03-11},
    langid = {english},
    eprinttype = {arxiv},
    eprint = {2603.10592 [cs]},
}

@article{lai_unified_2026,
    title = {A Unified View of Drifting and Score-Based Models},
    url = {http://arxiv.org/abs/2603.07514},
    publisher = {{arXiv}},
    journal = {arXiv preprint arXiv:2603.07514},
    author = {Lai, Chieh-Hsin and Nguyen, Bac and Murata, Naoki and Takida, Yuhta and Uesaka, Toshimitsu and Mitsufuji, Yuki and Ermon, Stefano and Tao, Molei},
    urldate = {2026-03-24},
    date = {2026-03-19},
    year = {2026},
    langid = {english},
    eprinttype = {arxiv},
    eprint = {2603.07514 [cs]},
}

@article{turan_generative_2026,
    title = {Generative Drifting is Secretly Score Matching: a Spectral and Variational Perspective},
    url = {http://arxiv.org/abs/2603.09936},
    shorttitle = {Generative Drifting is Secretly Score Matching},
    journal = {arXiv preprint arXiv:2603.09936},
    publisher = {{arXiv}},
    author = {Turan, Erkan and Ovsjanikov, Maks},
    urldate = {2026-05-01},
    date = {2026-03-10},
    year = 2026,
    langid = {english},
    eprinttype = {arxiv},
    eprint = {2603.09936 [cs]},
}

@article{lee_identifiability_2026,
    title = {Identifiability and Stability of Generative Drifting with Companion-Elliptic Kernel Families},
    url = {https://arxiv.org/abs/2604.24196},
    journal = {arXiv preprint arXiv:2604.24196},
    publisher = {{arXiv}},
    author = {Lee, Hak Geun},
    year = {2026},
    date = {2026-04-27},
    eprinttype = {arxiv},
    eprint = {2604.24196 [stat.ML]},
}

@article{he_sinkhorn_2026,
    title = {Sinkhorn-Drifting Generative Models},
    url = {https://arxiv.org/abs/2603.12366},
    journal= {arXiv preprint arXiv:2603.12366},
    publisher = {{arXiv}},
    author = {He, Pengfei and Khangaonkar, Omkar and Pirsiavash, Hamed and Bai, Yu and Kolouri, Soheil},
    year = {2026},
    eprinttype = {arxiv},
    eprint = {2603.12366 [cs.LG]},
}

@inproceedings{paszke_pytorch_2019,
 author = {Paszke, Adam and Gross, Sam and Massa, Francisco and Lerer, Adam and Bradbury, James and Chanan, Gregory and Killeen, Trevor and Lin, Zeming and Gimelshein, Natalia and Antiga, Luca and Desmaison, Alban and Kopf, Andreas and Yang, Edward and DeVito, Zachary and Raison, Martin and Tejani, Alykhan and Chilamkurthy, Sasank and Steiner, Benoit and Fang, Lu and Bai, Junjie and Chintala, Soumith},
 booktitle = {Advances in Neural Information Processing Systems},
 pages = {},
 title = {{PyTorch}: An Imperative Style, High-Performance Deep Learning Library},
 url = {https://proceedings.neurips.cc/paper_files/paper/2019/file/bdbca288fee7f92f2bfa9f7012727740-Paper.pdf},
 volume = {32},
 year = {2019}
}

@misc{jax2018github,
  author = {James Bradbury and Roy Frostig and Peter Hawkins and Matthew James Johnson and Yash Katariya and Chris Leary and Dougal Maclaurin and George Necula and Adam Paszke and Jake Vander{P}las and Skye Wanderman-{M}ilne and Qiao Zhang},
  title = {{JAX}: composable transformations of {P}ython+{N}um{P}y programs},
  url = {http://github.com/jax-ml/jax},
  version = {0.3.13},
  year = {2018},
}

@article{kingma2017adam,
      title={Adam: A Method for Stochastic Optimization}, 
      author={Diederik P. Kingma and Jimmy Ba},
      journal={arXiv preprint arXiv:1412.69800},
      year={2017},
      eprint={1412.6980},
      archivePrefix={arXiv},
      primaryClass={cs.LG},
      url={https://arxiv.org/abs/1412.6980}, 
}

@INPROCEEDINGS{SonZhaSmoGreetal08,
  author = {L. Song and X. Zhang and A. J. Smola and A. Gretton and B. Sch{\"o}lkopf},
  title = {Tailoring density estimation via reproducing kernel moment matching},
  booktitle = {Proceedings of the 25th International Conference on Machine Learning},
  year = {2008},
  editor = {W. W. Cohen and A. McCallum and S. Roweis},
  pages = {992-999},
  address = {New York},
  publisher = {ACM Press},
  location = {Helsinki, Finland},
  url = {http://icml2008.cs.helsinki.fi/papers/icml2008proceedings.pdf}
}

@article{SimSch18,
  title = {Kernel Distribution Embeddings: Universal Kernels, Characteristic Kernels and Kernel Metrics on Distributions},
  author = {Simon-Gabriel, C. J. and Sch{\"o}lkopf, B.},
  url = {https://jmlr.csail.mit.edu/papers/volume19/16-291/16-291.pdf},
  journal = {Journal of Machine Learning Research},
  volume = {19},
  number = {44},
  pages = {1--29},
  year = {2018}
}

@article{ho2022classifierfree,
      title={Classifier-Free Diffusion Guidance}, 
      author={Jonathan Ho and Tim Salimans},
      journal={arXiv preprint arXiv:2207.12598},
      year={2022},
      eprint={2207.12598},
      archivePrefix={arXiv},
      primaryClass={cs.LG},
      url={https://arxiv.org/abs/2207.12598}, 
}
}

\FloatBarrier
\newpage

\appendix
\section*{Impact Statement}\label{sec:impact}
This work is a theoretical analysis of training objectives for generative models, 
specifically reinterpreting and improving upon the drifting framework through the lens 
of conservative vector fields and kernel density estimation. It introduces no new 
datasets, large-scale pre-trained models, or deployment-ready systems.

\paragraph{Positive impacts.} By establishing a scalar \logkde{} loss equivalent to the 
drifting field, this work simplifies the implementation of one-step generative models 
and brings them closer to well-understood frameworks (Wasserstein gradient flows, score 
matching). This may lower the barrier for researchers to study, analyze, and improve 
such models, fostering scientific transparency and reproducibility.

\paragraph{Negative impacts.} Improvements to generative image models could, in 
principle, facilitate the creation of synthetic media (deepfakes) used for 
misinformation, fraud, or non-consensual imagery. However, this risk is indirect and 
shared with the broader generative modeling literature. The models trained in our 
experiments are small, applied to standard benchmarks (MNIST, Fashion-MNIST, ImageNet), 
and are substantially weaker than publicly available generative systems; we do not 
release pre-trained model weights.

\section*{Acknowledgments}
We kindly thank \citeauthor{deng_generative_2026} for making their codebase publicly available, allowing us to exactly replicate their experimental setting.

We greatly thank the Max Planck Computing and Data Facility and the Tübingen Machine Learning Cloud, DFG FKZ INST 37/1057-1 FUGG, for providing the compute resources for this research.
 
Funded by the European Union (REAL-RL, \href{https://doi.org/10.3030/101045454}{101045454}, WeatherGenerator, \href{https://doi.org/10.3030/101187947}{10118794}, and ANUBIS, \href{https://doi.org/10.3030/101123955}{101123955}). Views and opinions expressed are, however, those of the authors only and do not necessarily reflect those of the European Union or the European Research Council. Neither the European Union nor the granting authority can be held responsible for them.
 
This work was supported by the German Federal Ministry of Education and Research (BMBF): Tübingen AI Center, FKZ: 01IS18039A. Georg Martius is a member of the Machine Learning Cluster of Excellence, EXC number 2064/1 – Project number 390727645. The authors thank the International Max Planck Research School for Intelligent Systems (IMPRS-IS) for supporting Tim Weiland.

\section{Conservatism of the Drifting Field}

\subsection{Curl and Path Independence}
\begin{definition}[Curl]
\label{def:curl}
Let $h = \frac{n(n-1)}{2}$. The \textbf{curl} of a vector field $\V: \R^n \to \R^n$ is a map $\nabla \times \V: \R^n \to \R^h$ with
\begin{equation*}
    [\nabla \times \V]_{\idx(i,j)} :=
    (-1)^{i+j}\left(\frac{\del\V_i}{\del x_j} - \frac{\del\V_j}{\del x_i}\right) \quad \text{for} \quad 1 \leq i < j \leq n,
\end{equation*}
where $\idx: \{(i,j) \mid 1 \leq i < j \leq n\} \to \{1, \ldots, h\}$ is given by
\begin{equation*}
    \idx(i,j) = \frac{(n-i)(n-i-1)}{2} + n - j + 1.
\end{equation*}
In particular, for $n=2$, we recover the scalar-valued curl $\frac{\del \V_2}{\del x_1} - \frac{\del \V_1}{\del x_2}$.
\end{definition}

\paragraph{Proof for the Jacobian symmetry requirement.}
\label{app:conservative-proof}

We give a self-contained proof for Lemma~\ref{lem:jacobian_conservative}, i.e.

\textit{A vector field $\V: \R^n \to \R^n$ is conservative if and only if its Jacobian is symmetric.}

\begin{proof}
\underline{$(\Rightarrow)$:} If $\V = \nabla \Pot$, then $\frac{\del \V_i}{\del x_j} = \frac{\del^2 \Pot}{\del x_i \del x_j} = \frac{\del^2 \Pot}{\del x_j \del x_i} = \frac{\del \V_j}{\del x_i}$ by Schwarz's theorem.

\underline{$(\Leftarrow)$:} Assume the Jacobian of $\V$ is symmetric. Define $\Pot(\x) := \int_0^1 \V(t\x)^\top \x \, \d t$. Then:
\begin{align*}
    \frac{\del\Pot}{\del x_i} &= \int_0^1 \left(
        \V_i(t\x) + \sum_{j=1}^n \frac{\del \V_j(t\x)}{\del x_i} x_j
    \right) \d t \\
    &= \int_0^1 \left(
        \V_i(t\x) + \sum_{j=1}^n t\, \frac{\del \V_i(\x)}{\del x_j}\bigg|_{\x=t\x} x_j
    \right) \d t \\
    &= \int_0^1 \left(
        \V_i(t\x) + t\, \frac{\d}{\d t} \V_i(t\x)
    \right) \d t
    = \int_0^1 \frac{\d}{\d t}\bigl[t\, \V_i(t\x)\bigr] \d t
    = \V_i(\x),
\end{align*}
where the second line uses the chain rule $\frac{\del\V_j(t\x)}{\del x_i}= t\, \frac{\del \V_j(\x)}{\del x_i}\bigg|_{\x=t\x}$ and the symmetry assumption $\frac{\del \V_j}{\del x_i} = \frac{\del \V_i}{\del x_j}$. Thus $\nabla \Pot = \V$.
\end{proof}
\subsection{Relation of conservatism of $\V^\text{unnorm.}_{p,q}$ and $\Vp{p}^\text{unnorm.}$ and $\Vn{q}^\text{unnorm.}$}
\label{sec:relation_conservative}

We show formally that conservatism of the composite unnormalized field
\begin{align*}
    \V^\text{unnorm.}_{p,q}[k](\x) &= \Vp{p}^\text{unnorm.}[k](\x) - \Vn{q}^\text{unnorm.}[k](\x),\\
    \quad \text{with} \quad
    \Vp{p}^\text{unnorm.}[k](\x) &:= \E_{\yp\sim p}[k(\x,\yp)(\yp-\x)],
\end{align*}
for all $p, q \in \P(\R^n)$ implies that each subcomponent must be individually conservative. The proof relies on the following lemma:

\begin{lemma}\label{lem:curl_invariant_zero}
    If the curl $\nabla\times\Vp{p}^\text{unnorm.}[k]$ is invariant to $p$, then $\nabla\times\Vp{p}^\text{unnorm.}[k] = \mathbf{0}$ for all $p \in \P(\R^n)$.
\end{lemma}
\begin{proof}
    Choose $p = \delta(\yp - \y_0)$, so $\Vp{p}^\text{unnorm.}[k](\x) = k(\x,\y_0)(\y_0 - \x)$ and
    \begin{equation*}
        [\nabla\times\Vp{p}^\text{unnorm.}[k](\x)]_{\idx(i,j)} = (-1)^{i+j}\left(
            \frac{\del k(\x,\y_0)}{\del x_j}(y_{0,i} - x_i)
            - \frac{\del k(\x,\y_0)}{\del x_i}(y_{0,j} - x_j)
        \right).
    \end{equation*}
    Setting $\y_0 = \x$ gives $\nabla\times\Vp{p}^\text{unnorm.}[k](\x) = \mathbf{0}$. Since for every $\x$ there exists a measure $\delta(\yp - \x)$ that annihilates the curl at $\x$, and the curl is $p$-invariant by assumption, it must vanish for all $p$ and all $\x$.
\end{proof}

\begin{theorem}\label{thm:subfield_conservative}
    If $\V^\text{unnorm.}_{p,q}[k]$ is conservative for all $p, q \in \P(\R^n)$, then $\Vp{p}^\text{unnorm.}[k]$ is conservative for all $p \in \P(\R^n)$, and likewise $\Vn{q}^\text{unnorm.}[k]$.
\end{theorem}
\begin{proof}
    Suppose for contradiction that $\V^\text{unnorm.}_{p,q}[k]$ is conservative for all $p, q$ but there exists $p$ such that $\Vp{p}^\text{unnorm.}[k]$ is not conservative. From $\nabla\times\V^\text{unnorm.}_{p,q}[k] = \mathbf{0}$ it follows that
    \begin{align*}
        &\mathbf{0} = \nabla\times\Vp{p}^\text{unnorm.}[k] - \nabla\times\Vn{q}^\text{unnorm.}[k] \\
        \quad\Longrightarrow\quad
        &\nabla\times\Vn{q}^\text{unnorm.}[k] = \nabla\times\Vp{p}^\text{unnorm.}[k] \neq \mathbf{0}
        \quad \text{for all } q \in \P(\R^n).
    \end{align*}
    Thus $\nabla\times\Vn{q}^\text{unnorm.}[k]$ is invariant to $q$ and non-zero. By Lemma~\ref{lem:curl_invariant_zero}, however, a kernel subcomponent with $q$-invariant curl must have zero curl — a contradiction. The same argument with $p$ and $q$ swapped yields the result for $\Vn{q}^\text{unnorm.}[k]$.
\end{proof}

Consequently, to establish nonconservatism of $\V^\text{unnorm.}_{p,q}[k]$ it suffices to show that $\Vp{p}^\text{unnorm.}[k]$ or $\Vn{q}^\text{unnorm.}[k]$ is nonconservative for some $p$ or $q$, respectively.

\subsection{Nonconservatism in Arbitrary Dimensions}
\label{sec:nonconservatism-arbitrary-dim}

We show that a single counterexample in $\R^2$ suffices to establish nonconservatism in all dimensions $n \geq 2$.

\begin{lemma}[Restriction preserves conservatism]
\label{lem:restriction-conservative}
Let $f: \R^{n+1} \to \R^{n+1}$ be conservative with potential $\Pot$, and fix $c \in \R$. Define $g: \R^n \to \R^n$ by $g_i(\x) = f_i(x_1, \ldots, x_n, c)$ for $i = 1, \ldots, n$. Then $g$ is conservative with potential $\tilde\Pot(\x) = \Pot(x_1, \ldots, x_n, c)$.
\end{lemma}
\begin{proof}
Since $f = \nabla \Pot$, we have $g_i(\x) = \frac{\del \Pot}{\del x_i}\big|_{x_{n+1}=c} = \frac{\del}{\del x_i}\Pot(\x, c) = \frac{\del \tilde\Pot}{\del x_i}(\x)$, so $g = \nabla \tilde\Pot$.
\end{proof}

\begin{lemma}[Embedding of drift fields]
\label{lem:embedding-drift}
Let $\y_1, \ldots, \y_M \in \R^n$ and define $\tilde\y_m = (\y_m, 0) \in \R^{n+1}$. Denote by $\Vp{p}$ the drift sub-field in $\R^n$ with data $\{\y_m\}$ and by $\widetilde{\Vp{p}}$ the drift sub-field in $\R^{n+1}$ with data $\{\tilde\y_m\}$. Then for all $\x \in \R^n$:
\begin{equation*}
    \Vp{p}(\x) = \bigl(\widetilde{\Vp{p}}(\x, 0)\bigr)_{1:n},
\end{equation*}
where $(\cdot)_{1:n}$ denotes the first $n$ components.
\end{lemma}
\begin{proof}
At $\tilde\x = (\x, 0)$, we have $\|\tilde\x - \tilde\y_m\|^2 = \|\x - \y_m\|^2$, so $k(\tilde\x, \tilde\y_m) = k(\x, \y_m)$. The first $n$ components of the $(n{+}1)$-dimensional field are
\begin{equation*}
    \bigl(\widetilde{\Vp{p}}(\x, 0)\bigr)_{1:n}
    = \frac{
        \sum_m k(\x, \y_m)(\y_m - \x)
    }{
        \sum_m k(\x, \y_m)
    }
    = \Vp{p}(\x).
\end{equation*}
\end{proof}

\begin{proposition}
\label{prop:counterexample-induction}
If $\y_1, \ldots, \y_M \in \R^n$ is a counterexample to conservatism, then $\tilde\y_1, \ldots, \tilde\y_M \in \R^{n+1}$ is a counterexample in $\R^{n+1}$.
\end{proposition}
\begin{proof}
By contrapositive. Suppose $\widetilde{\Vp{p}}$ were conservative. By \Cref{lem:restriction-conservative}, restricting to $\{x_{n+1} = 0\}$ and taking the first $n$ components yields a conservative field. By \Cref{lem:embedding-drift}, this restricted field equals $\Vp{p}$, contradicting the assumption that $\{\y_m\}$ is a counterexample.
\end{proof}

Since we exhibit a counterexample in $n = 2$ (see \Cref{fig:spotlight}), it follows by induction that counterexamples to conservatism of normalized drift fields exist in all dimensions $n \geq 2$.

\subsection{Radiality is a Requirement for Conservatism}
\begin{lemma}
\label{lem:radial-kernel}
When the field $\V_{p,q}^\text{unnorm.}[k](\x)$ is conservative $k$ must be a \textbf{radial} kernel.
\end{lemma}

\begin{proof}
For the field to be generally conservative both components must be conservative individually. We thus know that the inside of the expectation $u(\x, \y) := k(\x, \y)(\y - \x)$ has to have a symmetric Jacobian: $\frac{\del}{\del x_j}[k(\x,\y)(\y_i - \x_i)] = \frac{\del}{\del x_i}[k(\x,\y)(\y_j - \x_j)]$ for all $i \neq j$. Rearranging this we get:
\begin{equation*}
    \frac{\frac{\del k}{\del x_j}}{\frac{\del k}{\del x_i}} =  \frac{\y_j - \x_j}{\y_i - \x_i},
\end{equation*}
meaning the gradient $\nabla_\x k(\x, \y)$ is parallel to $(\y - \x)$. For any fixed $\y$, this implies $k$ is constant on spheres centered at $\y$, i.e.\ $k(\x,\y) = \phi_\y(\|\x - \y\|^2)$ for some function $\phi_\y$. Since $k$ is symmetric ($k(\x,\y) = k(\y,\x)$), the profile $\phi_\y$ does not depend on $\y$, and we obtain $k(\x,\y) = \phi(\|\x - \y\|^2)$.
\end{proof}
\section{Detailed Kernel Derivations}
\label{app:kernel-derivations}

This appendix provides detailed derivations for the sharp and flat kernels of common radial kernels, supporting the summary given in \cref{tab:kernel-examples}.

\subsection{Gaussian Kernel}
\begin{corollary}
\label{corr:gaussian_kernel_sharp_proof}
For the Gaussian kernel, we have $k^\#(\x, \y) = \sigma^2 k(\x,\y)$ and $k^\flat(\x, \y) = \frac{1}{\sigma^2} k(\x,\y)$, so both are proportional to $k$.
\end{corollary}
\begin{proof}
Consider the Gaussian kernel with bandwidth $\sigma$:
\begin{equation*}
    k(\x,\y) = \exp\left(-\frac{\|\x-\y\|^2}{2\sigma^2}\right)
\end{equation*}

In radial form, this corresponds to $\phi(r) = \exp(-r/(2\sigma^2))$.

\emph{Computing the sharp:} Using the general formula for radial kernels:
\begin{align*}
    k^\#(\x, \y) &= \frac{1}{2}\int_{\|\x-\y\|^2}^\infty \phi(r) \, \d r
               = \frac{1}{2}\int_{\|\x-\y\|^2}^\infty e^{-r/(2\sigma^2)} \, \d r \\
               &= \frac{1}{2}\left[ -2\sigma^2 e^{-r/(2\sigma^2)} \right]_{\|\x-\y\|^2}^\infty
               = \sigma^2 \, k(\x,\y)
\end{align*}

\emph{Computing the flat:} We have $\phi'(r) = -\frac{1}{2\sigma^2} e^{-r/(2\sigma^2)}$. Then:
\begin{align*}
    k^\flat(\x, \y) &= -2\phi'(\|\x - \y\|^2)
                  = \frac{1}{\sigma^2} e^{-\|\x-\y\|^2/(2\sigma^2)}
                  = \frac{1}{\sigma^2} \, k(\x,\y)
\end{align*}

Thus both the sharp and flat are proportional to the original Gaussian kernel.
\end{proof}

\subsection{Laplacian Kernel}
\begin{corollary}
\label{corr:laplace_flat_sharp_proof}
For the Laplacian kernel $k(\x,\y) = \exp(-\|\x-\y\|/\sigma)$, the sharp is $k^\#(\x, \y) = \sigma (\|\x-\y\| + \sigma) k(\x,\y)$ and the flat is $k^\flat(\x, \y) = \frac{k(\x,\y)}{\sigma\|\x-\y\|}$.
\end{corollary}
\begin{proof}
The Laplacian kernel is $k(\x,\y) = \exp(-\|\x-\y\|/\sigma)$. In radial form, $\phi(r) = \exp(-\sqrt{r}/\sigma)$ where $r = \|\x-\y\|^2$.

\emph{Computing the sharp:} Using the general formula:
\begin{align*}
    k^\#(\x, \y) &= \frac{1}{2}\int_{\|\x-\y\|^2}^\infty \exp(-\sqrt{r}/\sigma) \, \d r
\end{align*}

Let $u = \sqrt{r}$, so $r = u^2$ and $\d r = 2u \, \d u$. When $r = \|\x-\y\|^2$, we have $u = \|\x-\y\|$:
\begin{align*}
    k^\#(\x, \y) &= \int_{\|\x-\y\|}^\infty u \exp(-u/\sigma) \, \d u
\end{align*}

Using integration by parts:
\begin{align*}
    &= \left[ -\sigma u \exp(-u/\sigma) \right]_{\|\x-\y\|}^\infty + \sigma \int_{\|\x-\y\|}^\infty \exp(-u/\sigma) \, \d u \\
    &= \sigma \|\x-\y\| \exp(-\|\x-\y\|/\sigma) + \sigma^2 \exp(-\|\x-\y\|/\sigma) \\
    &= \sigma (\|\x-\y\| + \sigma) k(\x,\y)
\end{align*}

\emph{Computing the flat:} For $\phi(r) = \exp(-\sqrt{r}/\sigma)$, we have:
\begin{equation*}
    \phi'(r) = -\frac{1}{2\sigma\sqrt{r}} \exp(-\sqrt{r}/\sigma)
\end{equation*}

Therefore:
\begin{align*}
    k^\flat(\x, \y) &= -2\phi'(\|\x - \y\|^2)
                  = \frac{1}{\sigma\|\x-\y\|} \exp(-\|\x-\y\|/\sigma)
                  = \frac{k(\x,\y)}{\sigma\|\x-\y\|}
\end{align*}
\end{proof}

\subsection{Rational Quadratic Kernel}
\begin{corollary}
\label{corr:rqk_flat_sharp_proof}
For the rational quadratic kernel $k(\x,\y) = (1 + \|\x-\y\|^2/\sigma^2)^{-2}$, the sharp is $k^\#(\x, \y) = \frac{\sigma^2}{2} k(\x,\y)^{1/2}$ and the flat is $k^\flat(\x, \y) = \frac{4}{\sigma^2} k(\x,\y)^{3/2}$.
\end{corollary}
\begin{proof}
In radial form, $\phi(r) = (1 + r/\sigma^2)^{-2}$.

\emph{Computing the sharp:}
\begin{align*}
    k^\#(\x, \y) &= \frac{1}{2}\int_{\|\x-\y\|^2}^\infty (1 + r/\sigma^2)^{-2} \, \d r
\end{align*}

Let $t = 1 + r/\sigma^2$, so $\d r = \sigma^2 \, \d t$:
\begin{align*}
    k^\#(\x, \y) &= \frac{\sigma^2}{2}\int_{1+\|\x-\y\|^2/\sigma^2}^\infty t^{-2} \, \d t
               = \frac{\sigma^2}{2} \left[ -t^{-1} \right]_{1+\|\x-\y\|^2/\sigma^2}^\infty \\
               &= \frac{\sigma^2}{2} (1 + \|\x-\y\|^2/\sigma^2)^{-1}
               = \frac{\sigma^2}{2} k(\x,\y)^{1/2}
\end{align*}

\emph{Computing the flat:} We have $\phi'(r) = -\frac{2}{\sigma^2}(1 + r/\sigma^2)^{-3}$. Then:
\begin{align*}
    k^\flat(\x, \y) &= -2\phi'(\|\x - \y\|^2)
                  = \frac{4}{\sigma^2}(1 + \|\x-\y\|^2/\sigma^2)^{-3}
                  = \frac{4}{\sigma^2} k(\x,\y)^{3/2}
\end{align*}
\end{proof}

\subsection[Half-Integer Mat\'ern Kernels]{Half-Integer Mat\'ern Kernels ($\nu = p + \tfrac{1}{2}$)}
\label{app:matern-derivation}
Let $p \in \mathbb{Z}_{\geq 0}$, $\nu = p+\tfrac{1}{2}$, and $c := \sqrt{2\nu}/\sigma$. The half-integer Mat\'ern kernel can be written without special functions as
\begin{equation*}
    k_\nu(\x,\y) = Q_p(c\|\x-\y\|)\,e^{-c\|\x-\y\|},
    \qquad
    Q_p(t) = \sum_{i=0}^{p}
    \frac{p!(p+i)!}{(2p)!\,i!\,(p-i)!}(2t)^{p-i}.
\end{equation*}
Thus $Q_0(t)=1$ gives the Laplacian kernel, $Q_1(t)=1+t$ gives Mat\'ern-$\tfrac{3}{2}$, and $Q_2(t)=1+t+t^2/3$ gives Mat\'ern-$\tfrac{5}{2}$.

\begin{proposition}[Sharp of half-integer Mat\'ern]
\label{prop:matern-half-sharp}
Let $p \in \mathbb{Z}_{\geq 0}$, $c>0$, and
\begin{equation*}
    k(\x,\y) = Q_p(c\|\x-\y\|)\,e^{-c\|\x-\y\|},
    \qquad \x,\y \in \R^d.
\end{equation*}
Then
\begin{equation*}
    k^\#(\x,\y) =
    \frac{2p+1}{c^2}\,Q_{p+1}(c\|\x-\y\|)\,e^{-c\|\x-\y\|}.
\end{equation*}
\end{proposition}
\begin{proof}
Write $R := \|\x-\y\|$ and $\phi(r) = Q_p(c\sqrt{r})e^{-c\sqrt{r}}$. \emph{Computing the sharp:} By \Cref{eq:sharp-flat-radial},
\begin{align*}
    k^\#(\x,\y)
    &= \frac{1}{2}\int_{R^2}^\infty Q_p(c\sqrt{r})e^{-c\sqrt{r}}\,\d r
     = \int_R^\infty u Q_p(cu)e^{-cu}\,\d u.
\end{align*}
With $t := cR$ and $s := cu$, this becomes
\begin{equation*}
    k^\#(\x,\y) = \frac{1}{c^2}\int_t^\infty s Q_p(s)e^{-s}\,\d s.
\end{equation*}
Write $Q_p(s)=\sum_{m=0}^{p} b_{p,m}s^m$, where
\begin{equation*}
    b_{p,m} =
    \frac{2^m p!(2p-m)!}{(2p)!\,(p-m)!\,m!}.
\end{equation*}
Repeated integration by parts gives
\begin{equation*}
    \int_t^\infty s^n e^{-s}\,\d s
    = n! e^{-t}\sum_{j=0}^n \frac{t^j}{j!}.
\end{equation*}
Therefore
\begin{align*}
    \int_t^\infty sQ_p(s)e^{-s}\,\d s
    &= e^{-t}\sum_{m=0}^{p}b_{p,m}(m+1)!
        \sum_{j=0}^{m+1}\frac{t^j}{j!} \\
    &= e^{-t}\sum_{j=0}^{p+1}
        \frac{1}{j!}\sum_{m=\max(0,j-1)}^{p}
        b_{p,m}(m+1)!\, t^j.
\end{align*}
We now identify this last polynomial with $(2p+1)Q_{p+1}(t)$. Let
\begin{equation*}
    B_j := \sum_{m=\max(0,j-1)}^{p} b_{p,m}(m+1)!.
\end{equation*}
The endpoint coefficient satisfies
\begin{equation*}
    B_{p+1}=b_{p,p}(p+1)!=(2p+1)(p+1)!b_{p+1,p+1}.
\end{equation*}
For $0\leq j\leq p$, direct substitution of the closed form for $b_{p,m}$ gives
\begin{equation*}
    (2p+1)\bigl(j!b_{p+1,j}-(j+1)!b_{p+1,j+1}\bigr)
    =
    \begin{cases}
        0, & j=0,\\
        j!b_{p,j-1}, & 1\leq j\leq p.
    \end{cases}
\end{equation*}
These are exactly the backward differences of $B_j$, since $B_0-B_1=0$ and $B_j-B_{j+1}=j!b_{p,j-1}$ for $1\leq j\leq p$. Hence $B_j=(2p+1)j!b_{p+1,j}$ for every $j$, proving
\begin{equation*}
    \int_t^\infty sQ_p(s)e^{-s}\,\d s
    = (2p+1)Q_{p+1}(t)e^{-t}.
\end{equation*}
For $p=0$ this reads $\int_t^\infty se^{-s}\,\d s=(1+t)e^{-t}$; for $p=1$ it reads $\int_t^\infty s(1+s)e^{-s}\,\d s=(t^2+3t+3)e^{-t}=3Q_2(t)e^{-t}$. Substituting $t=cR$ gives the claim.
\end{proof}

\begin{proposition}[Flat of half-integer Mat\'ern, $p\geq 1$]
\label{prop:matern-half-flat}
Let $p \in \mathbb{Z}_{\geq 1}$, $c>0$, and
\begin{equation*}
    k(\x,\y) = Q_p(c\|\x-\y\|)\,e^{-c\|\x-\y\|},
    \qquad \x,\y \in \R^d.
\end{equation*}
Then
\begin{equation*}
    k^\flat(\x,\y) =
    \frac{c^2}{2p-1}\,Q_{p-1}(c\|\x-\y\|)\,e^{-c\|\x-\y\|}.
\end{equation*}
The statement does not apply at $p=0$, where the Laplacian flat is singular at $\x=\y$.
\end{proposition}
\begin{proof}
Write again $R := \|\x-\y\|$, $u := c\sqrt{r}$, and $\phi(r)=Q_p(u)e^{-u}$. \emph{Computing the flat:} The chain rule gives
\begin{equation*}
    \phi'(r)
    = \frac{c}{2\sqrt{r}}\bigl(Q_p'(u)-Q_p(u)\bigr)e^{-u}.
\end{equation*}
It remains to compute $Q_p'(u)-Q_p(u)$. With $Q_p(u)=\sum_{m=0}^{p}b_{p,m}u^m$, the constant term vanishes because $b_{p,0}=b_{p,1}=1$. For $1\leq m\leq p-1$,
\begin{equation*}
    (m+1)b_{p,m+1}-b_{p,m}
    = -\frac{1}{2p-1}b_{p-1,m-1},
\end{equation*}
and the top coefficient satisfies
\begin{equation*}
    -b_{p,p}=-\frac{1}{2p-1}b_{p-1,p-1}.
\end{equation*}
Both identities follow by inserting the factorial formula for $b_{p,m}$. Hence
\begin{equation*}
    Q_p'(u)-Q_p(u)
    = -\frac{u}{2p-1}Q_{p-1}(u).
\end{equation*}
For example, $Q_1'(u)-Q_1(u)=-u$ and $Q_2'(u)-Q_2(u)=-\tfrac{u}{3}(1+u)$. Therefore
\begin{align*}
    \phi'(r)
    &= -\frac{c}{2\sqrt{r}}\frac{u}{2p-1}Q_{p-1}(u)e^{-u}
     = -\frac{c^2}{2(2p-1)}Q_{p-1}(c\sqrt{r})e^{-c\sqrt{r}}.
\end{align*}
Using $k^\flat(\x,\y)=-2\phi'(R^2)$ proves the formula.
\end{proof}

The two propositions are most transparent when restated in terms of the Mat\'ern kernel itself, with the length scale tracked explicitly. Write $k_{\nu,\sigma}$ for the half-integer Mat\'ern kernel of smoothness $\nu = p + \tfrac{1}{2}$ at length scale $\sigma$, so $c = \sqrt{2\nu}/\sigma$ in our earlier notation.

\begin{corollary}[Sharp and flat as Mat\'ern index shifts]
\label{cor:matern-half-shift}
For $\nu = p + \tfrac{1}{2}$ with $p \geq 0$,
\begin{equation}
\label{eq:matern-sharp-shift}
    k_{\nu,\sigma}^{\#}(\x, \y) = \sigma^2 \cdot k_{\nu+1,\, \sigma_\#}(\x, \y),
    \qquad \sigma_\# := \sqrt{\tfrac{\nu+1}{\nu}}\,\sigma.
\end{equation}
For $\nu = p + \tfrac{1}{2}$ with $p \geq 1$,
\begin{equation}
\label{eq:matern-flat-shift}
    k_{\nu,\sigma}^{\flat}(\x, \y) = \frac{\nu}{(\nu-1)\sigma^2} \cdot k_{\nu-1,\, \sigma_\flat}(\x, \y),
    \qquad \sigma_\flat := \sqrt{\tfrac{\nu-1}{\nu}}\,\sigma.
\end{equation}
\end{corollary}
\begin{proof}
With $c = \sqrt{2\nu}/\sigma$ and $\sigma_\# = \sqrt{(\nu+1)/\nu}\,\sigma$, observe that $\sqrt{2(\nu+1)}/\sigma_\# = c$. Hence the factor $Q_{p+1}(c\|\x-\y\|)\,e^{-c\|\x-\y\|}$ in \Cref{prop:matern-half-sharp} is exactly $k_{\nu+1, \sigma_\#}(\x, \y)$. The constant simplifies as $(2p+1)/c^2 = 2\nu \cdot \sigma^2 / (2\nu) = \sigma^2$ since $2\nu = 2p+1$. The flat case is identical with $\sigma_\flat = \sqrt{(\nu-1)/\nu}\,\sigma$ giving $\sqrt{2(\nu-1)}/\sigma_\flat = c$, and $c^2/(2p-1) = (2\nu/\sigma^2)/(2(\nu-1)) = \nu/((\nu-1)\sigma^2)$.
\end{proof}

The corollary makes precise the sense in which sharp and flat shift $\nu$ along the half-integer Mat\'ern family: the kernel order is shifted by $\pm 1$, but the length scale is simultaneously rescaled by $\sqrt{(\nu \pm 1)/\nu}$ to preserve the underlying decay rate $c$. As $\nu \to \infty$ this rescaling factor tends to $1$ (recovering the Gaussian fixed point), and at $\nu = \tfrac{1}{2}$ the flat side falls off the family because $\sigma_\flat \to 0$.

For Mat\'ern-$\tfrac{3}{2}$ ($p=1$, $c=\sqrt{3}/\sigma$), these propositions give
\begin{align*}
    k^\#(\x,\y)
    &= \frac{3}{c^2}\left(1+cR+\frac{c^2R^2}{3}\right)e^{-cR} \\
    &= \sigma^2\left(1+\frac{\sqrt{3}R}{\sigma}+\frac{R^2}{\sigma^2}\right)e^{-\sqrt{3}R/\sigma}, \\
    k^\flat(\x,\y) &= \frac{3}{\sigma^2}e^{-\sqrt{3}R/\sigma}.
\end{align*}
For Mat\'ern-$\tfrac{5}{2}$ ($p=2$, $c=\sqrt{5}/\sigma$),
\begin{align*}
    k^\flat(\x,\y)=\frac{5}{3\sigma^2}(1+cR)e^{-cR},
    \qquad
    k^\#(\x,\y)
    =\frac{5}{c^2}\left(1+cR+\frac{2c^2R^2}{5}+\frac{c^3R^3}{15}\right)e^{-cR},
\end{align*}
where $R=\|\x-\y\|$. Thus $Q_3(t)=1+t+2t^2/5+t^3/15$, as expected.

\paragraph{Mat\'ern smoothness ladder.}
For $\nu=p+\tfrac{1}{2}$, the sharp operation shifts the half-integer Mat\'ern family from $\nu$ to $\nu+1$ up to the factor $(2p+1)/c^2$ and the length-scale rescaling $\sigma \mapsto \sqrt{(\nu+1)/\nu}\,\sigma$, because the decay rate $c$ is held fixed. For $p\geq 1$, the flat operation shifts $\nu$ to $\nu-1$ up to the factor $c^2/(2p-1)$ and the corresponding rescaling $\sigma \mapsto \sqrt{(\nu-1)/\nu}\,\sigma$. The Gaussian is the $\nu\to\infty$ fixed point of the sharp and flat operations up to constants (\Cref{corr:gaussian_kernel_sharp_proof}), while the Laplacian ($p=0$) has no regular flat, consistent with the singularity in \Cref{corr:laplace_flat_sharp_proof}.

\section{Conservative Training}
\label{app:conservative_training}

We show that if $\V_{p,q}(\x)$ is conservative, i.e. if there exists a scalar potential $\Pot_{p,q}(\x): \R^n \to \R$ such that $\V_{p,q}(\x) = -\nabla_\x \Pot_{p,q}(\x)$, a simplified scalar training objective for training drifting models can be derived.  Starting from the definition of $\L^\text{drift}(\theta)$ in \cref{eq:drifting-loss-intro} and substituting in our assumption, we get:
\begin{align}
        \nabla_\theta \E_{\varepsilon}\left[\L^\text{drift}(\theta)\right]
        = &\;\E_{\varepsilon \sim p_\varepsilon}\left[
            \nabla_\theta \|f_\theta(\varepsilon) - \sg(
                f_{\theta}(\varepsilon) + \V_{p,q_\theta}(f_{\theta}(\varepsilon)
            ) \|^2
        \right]\nonumber\\
        = &\;-2\E_{\epsilon \sim p_\varepsilon}\left[
             \frac{\del f_{\theta}(\varepsilon)}{\del \theta}^\top \V_{p,q_\theta}(f_{\theta}(\varepsilon))
        \right]\nonumber\\\label{eq:pseudo_chain_rule}
        = &\;2\E_{\varepsilon \sim p_\varepsilon}\left[
             \frac{\del f_{\theta}(\varepsilon)}{\del \theta}^\top \nabla_\x \Pot_{p, q_\theta}(\x)\Big|_{\x = f_{\theta}(\varepsilon)}
        \right]
\end{align}
Note that the multi variate chain rule states for general $h: \R^n \to \R, g: \R^m \to \R^n$ that
\begin{equation*}
    \nabla_\x h(g(\x)) = \frac{\del g(\x)}{\del \x}^\top \nabla_\y h(\y)\Big |_{\y = g(\x)}.
\end{equation*}
If $\theta$ is assumed to be constant in the index of $\Pot_{p, q_\theta}$ (we denote this by $\Pot_{p,\sg(q_\theta)}$) the expression inside the expectation in \cref{eq:pseudo_chain_rule} is exactly the right hand side of the chain rule for $\nabla_\theta \Pot_{p,\sg(q_\theta)}(f_\theta(\epsilon))$, yielding:
\begin{align*}
    \nabla_\theta \E_{\varepsilon}\left[\L^\text{drift}(\theta)\right]
    = \E_{\varepsilon}\left[\nabla_\theta \Pot_{p,\sg(q_\theta)}(f_\theta(\epsilon))\right]
\end{align*}
After pulling out the gradient of the expectation, and applying the reparameterization trick in reverse we obtain:
\begin{align*}
        \nabla_\theta \E_{\varepsilon}\left[\L^\text{drift}(\theta)\right]
        = \;2\nabla_\theta \E_{\x\sim q_\theta}
          \left[\Pot_{p, \sg(q_\theta)}(\x)\right].
\end{align*}

\section{Squared MMD as a Stop-Gradient Objective}
\label{app:mmd-derivation}

We show how the squared MMD loss can be cast into the stop-gradient form used in the main text. The squared MMD between two distributions $p$ and $q$ with kernel $k$ is
\begin{equation*}
    \mathrm{MMD}^2_k(p, q_\theta) = \E_{p,p}[k(\y, \y')] - 2\,\E_{p,q_\theta}[k(\y, \x)] + \E_{q_\theta,q_\theta}[k(\x, \x')].
\end{equation*}
Calculating the gradient of the repulsion term $\E_{q_\theta, q_\theta}[k(\x, \x')]$ we get
\begin{align*}
    \nabla_\theta \E_{q_\theta, q_\theta}[k(\x, \x')]
    &= \E_{\varepsilon, \varepsilon'}\bigl[
        \nabla_\x k(\x, \x')^\top \nabla_\theta f_\theta(\varepsilon)
        + \nabla_{\x'} k(\x, \x')^\top \nabla_\theta f_\theta(\varepsilon')
    \bigr].
\end{align*}
By symmetry of $k$, i.e.\ $k(\x,\x') = k(\x', \x)$, and by the identical distribution of $\varepsilon$ and $\varepsilon'$, both terms have the same expectation. This thus can be expressed with the $\sg$ as:
\begin{align*}
    2 \nabla_\theta \E_{q_\theta, \sg(q_\theta)}[k(\x, \x')]
\end{align*}

Dropping the $p, p$ term and the scalar factor 2, we get the following by unifying the expectations over $q_\theta$:
\begin{align*}
    \L^{\text{MMD}^2}_{k}(\theta) &= \E_{\x \sim q_\theta}\bigl[
        \underbrace{\E_{\yn \sim \sg(q_\theta)}[k(\x, \yn)] - \E_{\yp \sim p}[k(\x, \yp)]}_{=:\;\Pot_{p,\sg(q_\theta)}^\text{MMD}[k](\x)}
    \bigr],
\end{align*}
which has a $\theta$-gradient equals $\frac{1}{2}\nabla_\theta \mathrm{MMD}^2_k(p, q_\theta)$. The inner function $\Pot_{p,\sg(q_\theta)}^\text{MMD}[k](\x)$ is the per-sample MMD potential whose $\nabla_\x$ gradient defines the MMD-induced field via \cref{eq:loss_identity}.

\section{Identifiability of Sharp-Normalized Equilibria}
\label{app:sharp-identifiability}

\begin{proof}[Proof of \Cref{prop:sharp-identifiability}]
If $p=q$, then $p_{\mathrm{KDE}}[k^\#]=q_{\mathrm{KDE}}[k^\#]$, so $\Vs_{p,q}=0$ by \Cref{eq:sharp-field}.
It remains to prove the converse.
Since $\Vs_{p,q}(\x)= -\nabla_\x \log(q_{\mathrm{KDE}}[k^\#](\x)/p_{\mathrm{KDE}}[k^\#](\x))$ and $\Vs_{p,q}(\x)=0$ for all $\x$, we have
\begin{equation*}
    \nabla_\x \log\left(
        \frac{q_{\mathrm{KDE}}[k^\#](\x)}
             {p_{\mathrm{KDE}}[k^\#](\x)}
    \right)=0
    \qquad \text{for all } \x.
\end{equation*}
The domain is connected, so the log-ratio is constant. Hence there exists $C>0$ such that
\begin{equation*}
    q_{\mathrm{KDE}}[k^\#](\x) = C\,p_{\mathrm{KDE}}[k^\#](\x)
    \qquad \text{for all } \x.
\end{equation*}
Integrating both sides over $\R^n$ and using Fubini gives
\begin{align*}
    \int q_{\mathrm{KDE}}[k^\#](\x)\,\d\x
    &=
    \int \int k^\#(\x,\y)\,\d q(\y)\,\d\x
    =
    \int c_\#\,\d q(\y)
    =
    c_\#,\\
    \int p_{\mathrm{KDE}}[k^\#](\x)\,\d\x
    &=
    c_\#.
\end{align*}
Therefore $c_\#=C c_\#$ and $C=1$. Thus $q_{\mathrm{KDE}}[k^\#]=p_{\mathrm{KDE}}[k^\#]$ pointwise. Since $k^\#$ is characteristic, the KDE map is injective on probability measures, and consequently $q=p$.
\end{proof}

\section{Wrong Drifting Normalization}
\label{app:wrong_drifting}
\begin{corollary}
\label{cor:wrong_drifting}
The pseudocode given by \citet{deng_generative_2026} in Algorithm~2, does not produce the original drifting field
\begin{align}
\V_{p,q}(\x) = \frac{1}{Z_p(\x)Z_q(\x)}\E_{p,q}\left[k(\x,\yp)k(\x,\yn)(\yp - \yn)\right]
\label{eq:correct_V}
\end{align}
, even if the additional normalization over the batch dimension is omitted. Here,
\begin{align}
\label{eq:mc_zp}
Z_p(\x) &= \E_p\left[k(\x, \yp)\right] \approx \frac{1}{\Npos} \sum_{j=1}^\Npos k(\x, \yp[j]) \\
Z_q(\x) &= \E_q\left[k(\x, \yn)\right] \approx \frac{1}{\Nneg} \sum_{j=1}^\Nneg k(\x, \yn[j])
\label{eq:mc_zq}
\end{align}
denote the drift normalization factors that are computed using the Monte Carlo method. Instead, for the common case $\Npos = \Nneg = N$, the algorithm from \citet{deng_generative_2026} produces the field
\begin{align}
\bar{\V}_{p,q}(x) = \frac{1}{(Z_p(\x) + Z_q(\x))^2}\E_{p,q}\left[k(\x,\yp)k(\x,\yn)(\yp - \yn)\right] .
\end{align}

\end{corollary}
\begin{proof}
We proceed by analyzing the code line-by-line.  
\begin{lstlisting}
# x: [N, D]
# y_pos: [N_pos, D]
# y_neg: [N_neg, D]
# T: temperature

# compute pairwise distance
dist_pos = cdist(x, y_pos)  # [N, N_pos]
dist_neg = cdist(x, y_neg)  # [N, N_neg]

# ignore self (if y_neg is x)
dist_neg += eye(N) * 1e6

# compute logits
logit_pos = -dist_pos / T
logit_neg = -dist_neg / T

# concat for normalization
logit = cat([logit_pos, logit_neg], dim=1)
\end{lstlisting}

The first part calculates the logits of the exponential kernel $\left[\vv{l}\right]_{i,j} = \log k(\x_i, \y_j)$ between generated samples $\x_i$ and positive and negative samples $\yp[j]$, $\yn[j]$, respectively. Here, we use $\y$ (without plus or minus) to denote the concatenated target matrix $\y = \left(\yp[1], \dots, \yp[\Npos], \yn[1], \dots, \yn[\Nneg]\right)$.

\begin{lstlisting}
# normalize (normalization over batch dim has been omitted!)
A_row = logit.softmax(dim=-1)
A = A_row
\end{lstlisting}

Next, the softmax operation is applied across the $j$-dimension to calculate the kernels and normalize simultaneously in a numerical stable fashion. Assuming that the Monte Carlo approximations in \eqref{eq:mc_zp} and \eqref{eq:mc_zq} hold exactly, i.e. that $Z_p(\x) = \frac{1}{\Npos} \sum_{n=1}^\Npos k(\x, \yp[n])$ and $Z_q(\x) = \frac{1}{\Nneg} \sum_{n=1}^\Nneg k(\x, \yn[n])$, we obtain

\begin{align}
  \left[\vv{A}\right]_{i,j} &= \frac{\exp(\vv{l}_{i,j})}{\sum_{n=1}^{\Npos + \Nneg}{ \exp(\vv{l}_{i,n})}} = \frac{k(\x_i, \y_j)}{\sum_{n=1}^{\Npos}{k(\x_i, \yp[n])} + \sum_{n=1}^{\Nneg}{k(\x_i, \yn[n])}} \\
  &= \frac{k(\x_i, \y_j)}{\Npos Z_p(\x_i) + \Nneg Z_q(\x_i)} .
\end{align}

Notice that, due to the concatenation, the normalization factor is $\Npos Z_p(\x_i) + \Nneg Z_q(\x_i)$. If, instead, the softmax were to be taken across $\yp$ and $\yn$ independently, the softmax operation would correctly yield the normalization factors $\Npos Z_p(\x_i)$ and $\Nneg Z_q(\x_i)$.

\begin{lstlisting}
# back to [N, N_pos] and [N, N_neg]
A_pos, A_neg = A[:, :N_pos], A[:, N_pos:]

# compute the weights
W_pos = A_pos  # [N, N_pos]
W_neg = A_neg  # [N, N_neg]
W_pos *= A_neg.sum(dim=1,keepdim=True)
W_neg *= A_pos.sum(dim=1,keepdim=True)
\end{lstlisting}

Next, the weights for the individual $\yp[j]$ and $\yn[j]$ are calculated. Assuming $\Npos = \Nneg = N$ from here on yields 

\begin{align}
      \left[\vv{W}_\text{pos}\right]_{i,j} &= \frac{k(\x_i, \yp[j])}{N(Z_p(\x_i) + Z_q(\x_i))} \sum_{n=1}^{N} \frac{k(\x_i, \yn[n])}{N(Z_p(\x_i) + Z_q(\x_i))} \\
      &= \frac{Z_q(\x_i) k(\x_i, \yp[j])}{N(Z_p(\x_i) + Z_q(\x_i))^2}
      \label{eq:wpos}
\end{align}
, for the positive weights. Likewise for $\vv{W}_\text{neg}$, we obtain

\begin{align}
      \left[\vv{W}_\text{neg}\right]_{i,j} = \frac{Z_p(\x_i) k(\x_i, \yn[j])}{N(Z_p(\x_i) + Z_q(\x_i))^2} .
      \label{eq:wneg}
\end{align}

\begin{lstlisting}
drift_pos = W_pos @ y_pos # [N, D]
drift_neg = W_neg @ y_neg # [N, D]

V = drift_pos - drift_neg
return V 
\end{lstlisting}
The final step involves taking a weighted sum over $\yp$ and $\yn$ and subtracting them from each other to obtain the drifting field $\bar{\V}_{p,q}$. Substituting \eqref{eq:wpos} and \eqref{eq:wneg} into this sum yields

\begin{align}
\bar{\V}_{p,q}(\x_i) &= \sum_{j=1}^N \left[\vv{W}_\text{pos}\right]_{i,j} \yp[j] - \sum_{j=1}^N \left[\vv{W}_\text{neg}\right]_{i,j} \yn[j] \\
&= \frac{Z_q(\x_i) \frac{1}{N} \sum_{j=1}^N k(\x_i,\yp[j])\yp[j] - Z_p(\x_i) \frac{1}{N} \sum_{j=1}^N k(\x_i,\yn[j])\yn[j]}{(Z_p(\x_i) + Z_q(\x_i))^2} 
\label{eq:vpq_wrong}
\end{align}

Finally, we can substitute \eqref{eq:mc_zp} and \eqref{eq:mc_zq} into \eqref{eq:vpq_wrong}:

\begin{align}
\bar{\V}_{p,q}(\x_i) &= \frac{\E_q\left[k(\x_i,\yn)\right] \E_p\left[k(\x_i, \yp) \yp \right] - \E_p\left[k(\x_i,\yp)\right] \E_q\left[k(\x_i, \yn) \yn \right]}{(Z_p(\x_i) + Z_q(\x_i))^2} .
\end{align}

As $\E_q\left[k(\x_i,\yn)\right]$ is constant w.r.t. $\yp$ and $\E_p\left[k(\x_i,\yp)\right]$ is constant w.r.t. $\yn$ and moreover $p \indep q$, we can take the expectation over the joint distribution $(\yp, \yn) \sim p(\yp)q(\yn)$ to obtain

\begin{align}
\bar{\V}_{p,q}(\x_i) &= \frac{1}{(Z_p(\x_i) + Z_q(\x_i))^2} \E_{p,q}\left[k(\x_i,\yp)k(\x_i,\yn)(\yp - \yn)\right] .
\end{align}

\end{proof}

\subsection{Naive Drifting Algorithm}
\Cref{alg:correct_drifting} implements the original drifting field as given by \eqref{eq:drift-field-components}.

\begin{algorithm}
\caption{Compute Drifting Field Naively (i.e. as in \cref{eq:drift-field-components})}
\label{alg:correct_drifting}
\begin{lstlisting}
def compute_V(x: "[N, D]", y_pos: "[N_pos, D]", y_neg: "[N_neg, D]", T: float):
    # compute pairwise distance
    dist_pos = cdist(x, y_pos) # [N, N_pos]
    dist_neg = cdist(x, y_neg) # [N, N_neg]
    
    # compute logits
    logit_pos = -dist_pos / T
    logit_neg = -dist_neg / T
    logit_neg.fill_diagonal_(-inf)

    # compute normalized kernels 
    # k(x,y+)/N*Z_p and k(x,y-)/N*Z_n
    A_pos = logit_pos.softmax(dim=-1) # [N, N_pos]
    A_neg = logit_neg.softmax(dim=-1) # [N, N_neg]

    # compute product
    # k(x,y+)k(x,y-)/(N^2*Z_p*Z_q)
    A = A_pos[:, :, None] * A_neg[:, None, :] # [N, N_pos, N_neg]   

    # compute weights
    # Z_q k(x,y+) / (N*Z_p*Z_q)
    # Z_p k(x,y-) / (N*Z_p*Z_q)
    W_pos = A.sum(dim=2) # [N, N_pos]
    W_neg = A.sum(dim=1) # [N, N_neg]

    # compute drift field V
    drift_pos = W_pos @ y_pos # [N, D]
    drift_neg = W_neg @ y_neg # [N, D]
    V = drift_pos - drift_neg
\end{lstlisting}
\end{algorithm}

\clearpage
\section{Additional Training Details}
\label{app:training_details}
This section provides additional details on our experimental setup, including our approach to class-conditional data loading (\cref{app:data_loading}) and feature normalization for kernel similarity computation (\cref{app:feat_norm}). We also provide details on how classifier-free guidance is integrated into the \logkde{} objective for ImageNet experiments (\cref{app:cfg}). \Cref{tab:exp_configs} summarizes the training and architectural configurations, along with the hardware requirements for reproducing our experiments. For MNIST and Fashion-MNIST, we compute evaluation metrics using the ADM evaluation suite~\cite{dhariwal2021diffusion}.

\begin{table*}
    \centering
    \caption{\textbf{Experiment configurations used in this work.}}
    \vspace{.5em}
    \label{tab:exp_configs}
    \tablestyle{1pt}{1.0}
    \scriptsize
    \begin{tabular}{l |x{80}x{80}|x{130}}
    \toprule
    & \textbf{MNIST} & \textbf{Fashion-MNIST} & \textbf{ImageNet (Ablation)}  \\
    
    \midrule
    \rowcolor[gray]{0.9} \multicolumn{4}{l}
    {\textit{\textbf{Architecture}}} \\
    model & \multicolumn{2}{c|}{DiT-S/2} & DiT-B/2 \\
    input size & \multicolumn{2}{c|}{28$\times$28$\times$1} & {32$\times$32$\times$4}  \\
    patch size & \multicolumn{2}{c|}{2$\times$2} & 2$\times$2 \\
    feature dim & \multicolumn{2}{c|}{384} & 768 \\
    mlp dim & \multicolumn{2}{c|}{1536} & 3072 \\
    depth & \multicolumn{2}{c|}{12} & 12 \\
    register tokens & \multicolumn{2}{c|}{8} & 16 \\
    style embedding tokens & \multicolumn{2}{c|}{32} & 32 \\
    noise input & \multicolumn{2}{c|}{$\mathcal{N}(0,1)$} & $\mathcal{N}(0,1)$ \\
    
    \midrule
    \rowcolor[gray]{0.9} \multicolumn{4}{l}{\textit{\textbf{Optimizer}}} \\
    optimizer & \multicolumn{2}{c|}{AdamW ($\beta_1$ = 0.95, $\beta_2$ = 0.999)} & {AdamW ($\beta_1$ = 0.9, $\beta_2$ = 0.95)} \\
    learning rate & \multicolumn{2}{c|}{2e-4} & 2e-4 \\
    weight decay & \multicolumn{2}{c|}{1e-4} & 1e-4 \\
    cosine annealing & \multicolumn{2}{c|}{yes} & no \\
    no-op steps & \multicolumn{2}{c|}{300} & 0 \\
    ramp-up steps & \multicolumn{2}{c|}{300} & 5000 \\
    gradient clip & \multicolumn{2}{c|}{2.0} & 2.0 \\
    training steps & 12000 & 16000 & 30000 \\
    EMA decay & \multicolumn{2}{c|}{no EMA} & 0.999 \\
    
    \midrule
    \rowcolor[gray]{0.9} \multicolumn{4}{l}{\textit{\textbf{Loss}}} \\
    kernel widths & \multicolumn{2}{c|}{single scale} & single scale \\
    generation space & \multicolumn{2}{c|}{pixel} & SD-VAE~\citep{rombach2022highres, stabilityai/sd-vae-ft-mse}, same as~\citep{deng_generative_2026} \\
    loss space & \multicolumn{2}{c|}{pixel} & SD-VAE \& latent-MAE-256~\citep{deng_generative_2026} \\
    feature normalization & \multicolumn{2}{c|}{top-k EMA per class (k = 5, $\beta$ = 0.99)} & same as~\citep{deng_generative_2026} \\
    field / gradient normalization & \multicolumn{2}{c|}{no} & yes \\
    CFG & \multicolumn{2}{c|}{no} & $\alpha \in \left[1, 4\right],\ p(\alpha) \propto \alpha^{-3}$\\

    \midrule
    \rowcolor[gray]{0.9} \multicolumn{4}{l}{\textit{\textbf{Batch \& Training}}} \\
    class labels $N_\text{c}$ & \multicolumn{2}{c|}{5} & 64 \\
    positive samples $N_\text{pos}$ & \multicolumn{2}{c|}{96} & 64 \\
    generated samples $N_\text{neg}$ & \multicolumn{2}{c|}{96} & 64  \\
    effective batch $B$ ($N_\text{c}{\times}N_\text{neg}$) & \multicolumn{2}{c|}{480} & 4096 \\
    memory bank & \multicolumn{2}{c|}{no, specialized sampler instead} & yes (pos. = 1024 / node, neg. = 8000 / node)\\

    \midrule
    \rowcolor[gray]{0.9} \multicolumn{4}{l}{\textit{\textbf{Infrastructure}}} \\
    codebase & \multicolumn{2}{c|}{own implementation in PyTorch~\citep{paszke_pytorch_2019}} & \citet{deng_generative_2026} implementation in JAX~\citep{jax2018github} \\
    base config & \multicolumn{2}{c|}{-} & \emph{ablation default} \\
    nodes & \multicolumn{2}{c|}{1} & 1
    \\
    GPUs (total) & \multicolumn{2}{c|}{1x A100} & 8x H200 \\
    peak VRAM (per GPU) & \multicolumn{2}{c|}{31GB} & 118GB \\
    training time & \multicolumn{2}{c|}{1h} & 11h \\

    \bottomrule
    \end{tabular}
\end{table*}

\subsection{Class-Conditional Data Loading}
\label{app:data_loading}

Class-conditional loss computation requires a sufficient number of samples per class in every minibatch. With uniform sampling from the unconditional data distribution, however, a fixed batch size yields only a handful of samples per class. To address this, \citet{deng_generative_2026} introduce a \emph{sample queue}: a fixed-size reservoir buffer maintained per class, from which a fixed number of classes and, subsequently, a fixed number of samples per class are drawn each iteration.

This approach carries three notable drawbacks. First, the memory overhead can be substantial. For pixel-space generation on ImageNet, storing 1{,}024 samples per class requires at least 201GB per host. Second, a warmup period is needed to populate the queues before training can proceed effectively. Third, due to the finite buffer size, sampling may not perfectly reflect the true class-conditional data distribution.

For our experiments on MNIST and Fashion-MNIST, we eliminate the sample queue entirely with a lightweight class-conditional batch sampler. At the start of training, we precompute an index list for each class. During training, we uniformly sample a fixed number of classes and draw sample indices directly from the corresponding lists. This resolves all three issues at no memory cost, with no warmup requirement, and with exact uniformity over the class-conditional distribution. In PyTorch~\citep{paszke_pytorch_2019}, the approach reduces to a simple custom \texttt{BatchSampler} implementation.

\subsection{Feature Normalization}
\label{app:feat_norm}

Feature vectors can differ vastly in magnitude and pairwise distance across datasets and latent spaces, making it difficult to select a sensible kernel bandwidth in a general manner. Rather than tuning the kernel width per dataset or space, one can normalize the feature vectors themselves, leaving the loss formulation and all other hyperparameters unchanged. By absolute homogeneity of norms,
\begin{equation}
    \|\alpha\,\mathbf{x} - \alpha\,\mathbf{y}\| = \alpha\,\|\mathbf{x} - \mathbf{y}\|, \quad \alpha \in \mathbb{R}^+,
\end{equation}
so normalization factor $\alpha$ directly controls the effective kernel bandwidth for radial kernels.

\citet{deng_generative_2026} normalize each batch by the average pairwise distance over all samples, both real and generated. This makes the effective bandwidth dependent on the current pushforward distribution $q_\theta$ and introduces high variance from batch-level estimation.

For our MNIST and Fashion-MNIST experiments, we adopt a more stable alternative inspired by \citet{kynkaanniemi2019improved}. We estimate local manifold structure via the mean $k$-nearest-neighbor distance within each batch, tracked as an exponential moving average (EMA) per class. Let $\mathcal{N}(\mathbf{x})$ denote the $k$ nearest neighbors of $\mathbf{x}$ among $N$ real samples from class $c$. The EMA update with momentum $\beta \in \left[0, 1\right]$ is
\begin{equation}
    d_t = \beta\, d_{t-1} + (1-\beta)\frac{1}{Nk} \sum_{i=1}^{N} \sum_{\mathbf{x}' \in \mathcal{N}(\mathbf{x}_i)} \| \mathbf{x}_i - \mathbf{x}'\|,
\end{equation}
and both real and generated features are normalized by $d_t$ before loss computation, $\hat{\mathbf{x}} := \mathbf{x} / d_t$. This makes kernel widths interpretable in units of the mean nearest-neighbor distance. In practice, per-class EMAs are maintained by instantiating one state-full \emph{Loss} object per class. We initialize $d_0 = 1$ and skip the first 300 optimization steps to ensure $d_t$ is sufficiently stable. In the future, we plan to improve this by introducing a bias-correction term similar to \citep{kingma2017adam}.

\subsection{Classifier-Free Guidance}
\label{app:cfg}

Inspired by classifier-free guidance~(CFG)~\citep{ho2022classifierfree},
\citet{deng_generative_2026} augment the drifting objective with negative
samples drawn from the unconditional data distribution
$p_\text{data}(\cdot | \varnothing)$.
Concretely, negative samples are drawn from the mixture
\begin{equation}
    \tilde{q}_\theta(\cdot | c)
    \;=\;
    (1-\gamma)\,q_\theta(\cdot | c)
    \;+\;
    \gamma\, p_\text{data}(\cdot | \varnothing),
    \qquad \gamma \in [0,1],
\end{equation}
where $q_\theta(\cdot | c)$ is the class-conditional pushforward distribution.
Introducing the reparametrisation $\alpha = (1-\gamma)^{-1} \geq 1$, the
target distribution against which the generator is trained becomes
\begin{equation}
    \tilde{p}(\cdot | c)
    \;=\;
    \alpha\, p_\text{data}(\cdot | c)
    \;-\;
    (\alpha-1)\, p_\text{data}(\cdot | \varnothing),
\end{equation}
an analogue of the guidance-scaled distribution familiar from diffusion
models~\citep{ho2022classifierfree}.  In practice, $\alpha$ is sampled
uniformly from $[1, 3]$ and supplied as an additional input to the network.

To implement this efficiently, a fixed number $N_\text{unc}$ of unconditional
samples $\y^\text{unc}_j \sim p_\text{data}(\cdot | \varnothing)$ are appended
to each mini-batch and assigned scalar weights
\begin{equation}
    w \;=\; \frac{N_\text{neg} - 1}{N_\text{unc}}\,(\alpha - 1),
\end{equation}
where $N_\text{neg}$ is the number of generated (conditional) samples.  The
weighting ensures that the empirical mixture matches $\tilde{q}_\theta$ in
expectation for any sampled $\alpha$.

We adopt this scheme within our framework by replacing the expectation over
$q_\theta(\cdot | c)$ in the \logkde{} objective with an expectation over
$\tilde{q}_\theta(\cdot | c)$, yielding the CFG \logkde{} objective
\begin{equation}
\label{eq:cfg-loss}
    \mathcal{L}^{\mathrm{CFG}}_{p,q}(\theta)
    \;:=\;
    \mathbb{E}_{\mathbf{x} \sim q_\theta}
    \!\left[
        \log\!\left(
            \frac{\tilde{q}_{\mathrm{KDE}}[k^\#](\mathbf{x})}
                 {p_{\mathrm{KDE}}[k^\#](\mathbf{x})}
        \right)
    \right].
\end{equation}
The numerator is computed via a weighted \texttt{logsumexp} as
\begin{equation}
    \log \tilde{q}_{\mathrm{KDE}}[k^\#](\mathbf{x})
    \;=\;
    \log\!\left(
        \frac{1}{N_\text{neg} + N_\text{unc}\,w}
        \left[
            \sum_{i=1}^{N_\text{neg}} k^\#(\mathbf{x},\, \mathbf{y}^-_i)
            \;+\;
            w \sum_{j=1}^{N_\text{unc}} k^\#(\mathbf{x},\, \mathbf{y}^\text{unc}_j)
        \right]
    \right),
\end{equation}
where $\mathbf{y}^-_i \sim q_\theta(\cdot | c)$ are conditional negative
samples and $\mathbf{y}^\text{unc}_j \sim p_\text{data}(\cdot | \varnothing)$
are unconditional samples.

\clearpage
\section{Additional Results}
\label{app:additional_results}

We provide qualitative results for image generation in this section. \Cref{fig:samples_imagenet_abl_laplsharp} shows uncurated samples from our ImageNet experiment using \logkde{}-Laplace$^\#$ as the training objective. Note that these results are from ablation runs after only 30{,}000 training steps and are not directly comparable to the state-of-the-art runs from \citet{deng_generative_2026} that use bigger models and train for 200{,}000 steps.  \Cref{fig:samples_mnist} compares generations from different drift-based approaches on MNIST and Fashion-MNIST.

\begin{figure}[h]
    \centering
    \includegraphics[width=0.9\textwidth]{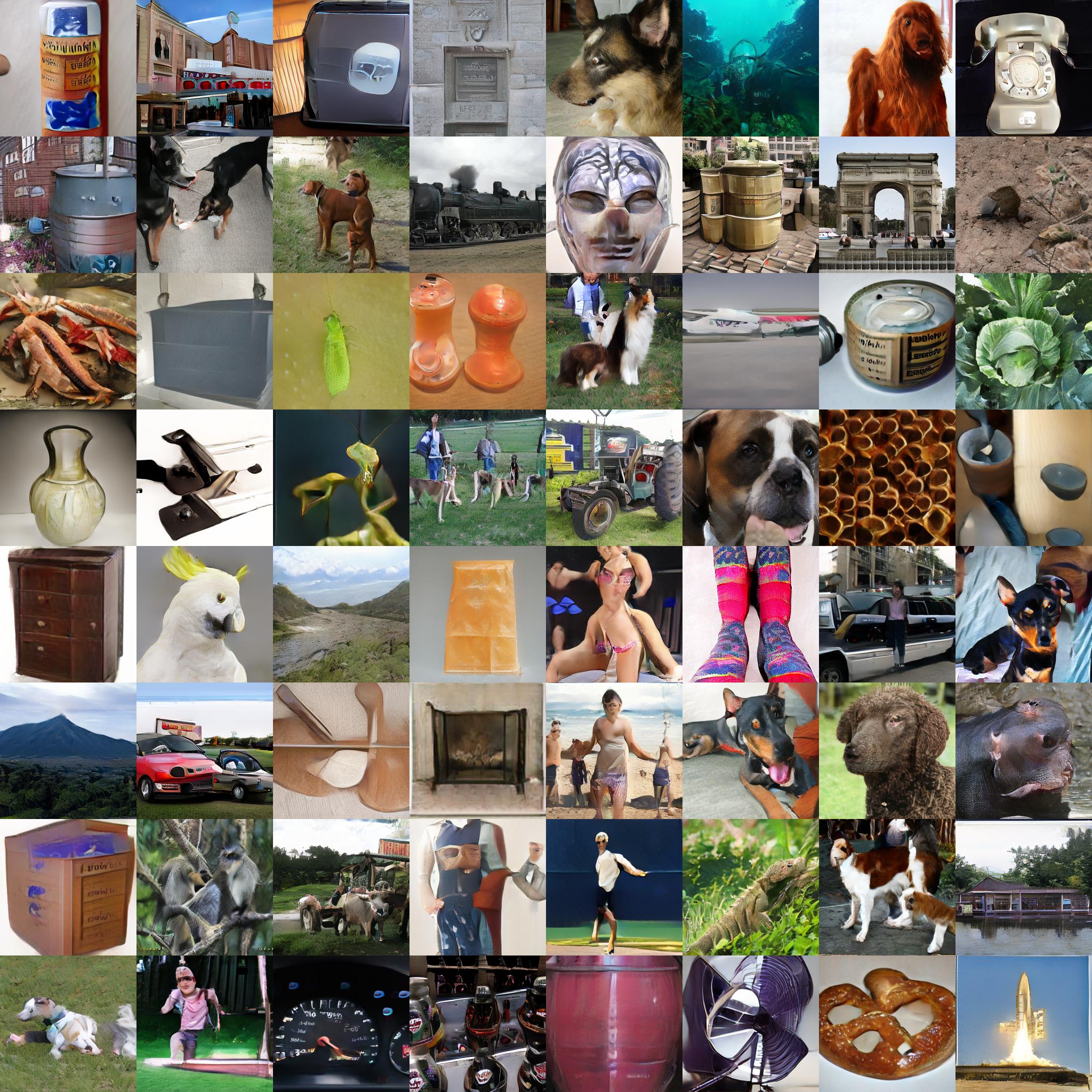}
    \caption{\textbf{Uncurated images generated on ImageNet.} Results after training for 30{,}000 steps with the \laplsharp{} objective with a kernel width of $0.3$ (FID-50k $= 9.35$).}
    \label{fig:samples_imagenet_abl_laplsharp}
\end{figure}

\begin{figure}[h]
    \centering
    \includegraphics[width=\textwidth]{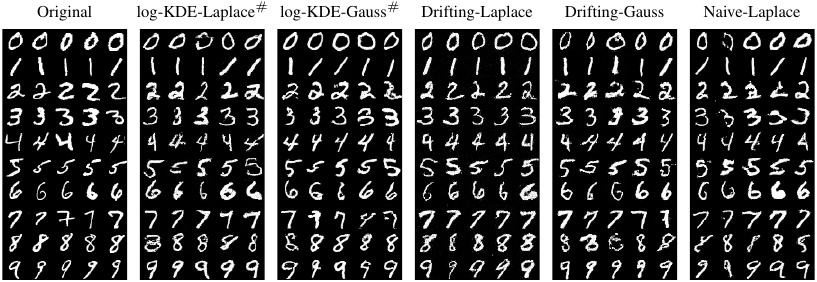}\\\vspace{1cm}
    \includegraphics[width=\textwidth]{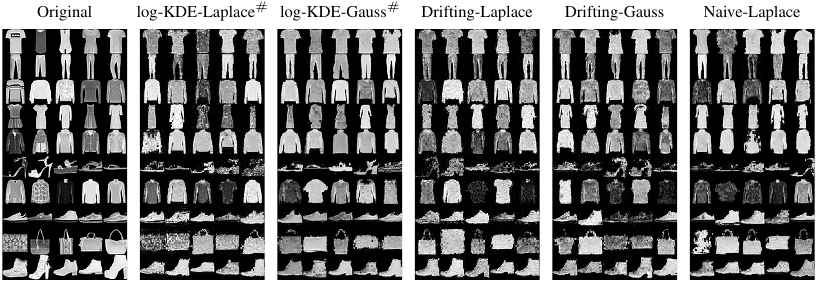}
    \caption{\textbf{Uncurated generated images on MNIST and Fashion-MNIST.} Real and generated samples for MNIST (top) and Fashion-MNIST (bottom) across different drift variants and \logkde{}. Images were generated using the kernel width yielding the lowest FID score for each variant (see \cref{fig:fid,tab:quant_results}).}
    \label{fig:samples_mnist}
\end{figure}

\end{document}